\documentclass[letterpaper]{article} 
\usepackage[]{aaai2026}  
\usepackage{graphicx}
\usepackage{subcaption}  
\usepackage{tikz}
\usetikzlibrary{positioning,matrix,arrows.meta}
\usepackage{amsthm}
\usepackage{multirow}
\usepackage{booktabs} 
\usepackage{times}  
\usepackage{helvet}  
\usepackage{courier}  
\usepackage{amsfonts} 
\usepackage[hyphens]{url}  
\usepackage{graphicx} 
\usepackage{amsmath}
\usepackage{natbib}  
\usepackage{caption} 
\usepackage{algorithm}
\usepackage{algorithmic}
\usepackage{newfloat}
\usepackage{listings}
\usepackage{float}
\usepackage{enumitem}

\newtheorem{theorem}{Theorem}
\newtheorem{lemma}[theorem]{Lemma} 

\newcommand{\triangleq}{\stackrel{\triangle}{=}} 

\DeclareCaptionStyle{ruled}{labelfont=normalfont,labelsep=colon,strut=off} 
\usepackage{booktabs}
\floatstyle{ruled}
\newfloat{listing}{tb}{lst}{}
\floatname{listing}{Listing}

\title{Beyond Training-time Poisoning: Component-level and Post-training Backdoors in Deep Reinforcement Learning}

\author{
\textnormal{
\begin{tabular}{c}
\textbf{Sanyam Vyas\textsuperscript{1,2}, Alberto Caron\textsuperscript{2}, Chris Hicks\textsuperscript{2}, Pete Burnap\textsuperscript{1}, Vasilios Mavroudis\textsuperscript{2}} \\
\textsuperscript{1}Cardiff University \\
\textsuperscript{2}The Alan Turing Institute \\
vyass3@cardiff.ac.uk, acaron@turing.ac.uk, c.hicks@turing.ac.uk, \\ burnapp@cardiff.ac.uk, vmavroudis@turing.ac.uk
\end{tabular}
}
}

\begin{document}

\maketitle

\begin{abstract}

Deep Reinforcement Learning (DRL) systems are increasingly used in safety-critical applications, yet their security remains severely underexplored. This work investigates backdoor attacks, which implant hidden triggers that cause malicious actions only when specific inputs appear in the observation space. Existing DRL backdoor research focuses solely on training-time attacks requiring full adversarial access to the training pipeline. In contrast, we reveal critical vulnerabilities across the DRL supply chain where backdoors can be embedded with significantly reduced adversarial privileges. We introduce two novel attacks: (1) TrojanentRL, which exploits component-level flaws to implant a persistent backdoor that survives full model retraining; and (2) InfrectroRL, a post-training backdoor attack which requires no access to training, validation, or test data. Empirical and analytical evaluations across six Atari environments show our attacks rival state-of-the-art training-time backdoor attacks while operating under much stricter adversarial constraints. We also demonstrate that InfrectroRL further evades two leading DRL backdoor defenses. These findings challenge the current research focus and highlight the urgent need for robust defenses.

\end{abstract}

\section{Introduction}

Deep Reinforcement Learning (DRL) delivers critical capabilities in safety-sensitive domains including autonomous vehicles~\cite{fayjie2018driverless}, nuclear fusion control~\cite{degrave2022magnetic}, cyber defense~\cite{vyas2025towards}, and drug 
discovery~\cite{tan2022reinforcement}, yet introduces serious security vulnerabilities to adversarial attacks during training and deployment. Compromised agents risk severe consequences~\cite{pattanaik2017robust}, making robust defenses essential.

Backdoor attacks compromise DRL agents through trigger-conditional malicious behavior while preserving normal performance on benign inputs. Current DRL backdoor research~\cite{rathbun2024adversarial,rathbun2024sleepernets,cui2023badrl,wang2021backdoorl} focuses narrowly on attacks requiring excessive adversary privilege, overlooking critical threats across the DRL supply chain. Moreover, existing methods demand impractical capabilities: infiltrating secure training pipelines, reverse-engineering proprietary codebases, developing undetectable attack scripts, and unrealistic requirements like direct RAM manipulation or full state-representation control, rendering them highly impractical beyond academic settings.

This work shifts focus from conventional training-time attacks to component-level and post-training backdoors. Our proposed attacks, \textbf{TrojanentRL} and \textbf{InfrectroRL}, achieve superior effectiveness and evasiveness with substantially reduced adversarial access compared to existing literature.

Inspired by threat models in~\cite{langford2024architectural,bober2023architectural}, \textit{TrojanentRL} embeds a backdoor in the DRL rollout buffer, achieving superior stealth under these assumptions~\cite{langford2024architectural,gu2023mamba}. Crucially, \textit{under the same assumptions}, TrojanentRL remains effective against all retraining and fine-tuning DRL backdoor defenses~\cite{chen2023bird,yuanshine}.

\textit{InfrectroRL} advances DRL backdoor threat models through direct, data-free modification of \textit{pretrained} model parameters~\cite{liu2018trojaning,cao2024data}. By optimizing triggers to establish persistent backdoor pathways that influence sequential actions, this attack operates with minimal computational overhead by circumventing training requirements.

Following rigorous DRL security evaluation standards~\cite{kiourti2020trojdrl,cui2023badrl,rathbun2024sleepernets,bharti2022provable,chen2023bird,yuanshine}, we benchmark both attacks across six Atari environments using established backdoor metrics. For InfrectroRL, we further: (1) derive theoretical guarantees of evasive performance during benign operation, and (2) demonstrate robust effectiveness against state-of-the-art defenses~\cite{chen2023bird,yuanshine}, addressing a critical gap in prior literature. Our main contributions can be summarized as follows:

\begin{itemize}[left=1pt]
    \item We present a \textbf{new end-to-end threat model}, i.e., a DRL threat model spanning multiple supply chain stages, revealing vulnerabilities beyond training-time attacks.
    
    \item We present \textbf{TrojanentRL} and \textbf{InfrectroRL}, novel backdoor attacks that achieve superior empirical performance over existing DRL backdoor attacks, while operating under significantly reduced adversarial access assumptions.

    \item We provide \textbf{theoretical guarantees} on InfrectroRL’s evasiveness under benign operation and perform \textbf{rigorous validation} across six Atari environments for both attacks. We also illustrate InfrectroRL's ability to empirically evade two state-of-the-art DRL backdoor defenses~\cite{chen2023bird,yuanshine}.
\end{itemize}

    


\begin{figure*}
    \hspace{0.18\textwidth} 
    \includegraphics[width=0.70\textwidth]{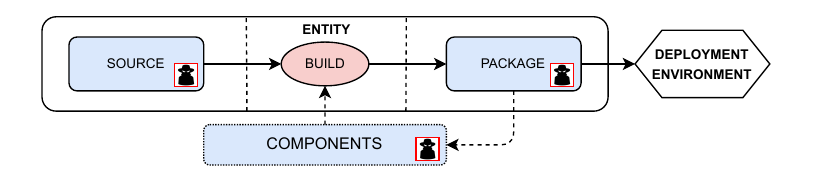}
    \caption{We ground our DRL threat model in the SLSA supply chain framework (\url{https://slsa.dev/spec/v0.1/threats}), categorizing AI supply chains into: \textbf{Source} (HuggingFace, TorchHub, GitHub), \textbf{Entity} integration (third-party developers, anonymous contributors, end users) of performance-enhancing \textbf{Components} (RL/ML libraries), \textbf{Packaging} (compressed artifacts), and \textbf{Deployment Environment}. Attacks (InfectroRL, TrojanentRL) exploit vulnerabilities from source integration through packaging.}
    \setlength{\belowcaptionskip}{-10pt}
    \label{fig:threat-model}
\end{figure*}

\section{Background}
\label{sec: background}

\subsection{Reinforcement Learning} 

Reinforcement Learning (RL) formalizes sequential decision-making via agent-environment interactions, modeled as a Markov Decision Process $(\mathcal{S}, \mathcal{A}, \mathcal{P}, \mathcal{R}, \gamma)$. At timestep $t$, the agent observes state $s_t \in \mathcal{S}$, selects action $a_t \in \mathcal{A}$ via policy $\pi: \mathcal{S} \rightarrow \mathcal{A}$, receives reward $r_t = \mathcal{R}(s_t, a_t)$, and transitions to $s_{t+1} \sim \mathcal{P}(\cdot|s_t, a_t)$. The objective is to maximize the expected discounted return:
\begin{equation}  
J(\pi) = \mathbb{E}_{a \sim \pi}\left[\sum_{t=0}^T \gamma^t r_t\right], \quad \gamma \in [0, 1),
\end{equation}
where discount factor $\gamma$ balances immediate versus future rewards. Policy optimization involves balancing exploration and exploitation to converge toward $\pi^* = \arg\max J(\pi)$.

Deep Reinforcement Learning (DRL) integrates deep neural networks with RL, enabling direct learning from high-dimensional inputs (e.g., images, sensor data). Unlike traditional methods (Monte Carlo, tabular Q-learning) that scale poorly, DRL algorithms like Deep Q-Networks (DQN)~\cite{mnih2013playing} and Proximal Policy Optimization (PPO)~\cite{PPO} achieve state-of-the-art performance by addressing: (1) sample efficiency in high-dimensional spaces, (2) stability during approximation, and (3) generalization across unseen states. Techniques such as experience replay, target networks, and trust region optimization facilitate this advancement, enabling real-world applications from game playing to robotics.

\subsection{Backdoor Attacks}
An emerging threat in the domain of DRL is represented by \textit{backdoor} attacks --- also referred to as \textit{trojan}~\cite{ahmed2024deep}. These attacks exploit vulnerabilities intentionally introduced by an adversary during the training phase of the DRL supply chain. Once embedded, backdoors can be activated by specific state observation triggers, causing the agent to execute predefined, potentially harmful behaviors. Formally, a triggered state can be represented as $\Tilde{s} := s + \delta$, where $s \in \mathcal{S}$ is the original state and $\delta$ is an adversarial perturbation. The adversary formulates the attack, generating $\Tilde{s}$ according to equation:
\begin{equation}
   \centering
   \Tilde{s} = (1-m) \circ s + m \circ \Delta ~ , \label{eq:attack-formulation}
\end{equation}
where $m$ and $\Delta$ are matrices that define the position mask and the value of the trigger $\delta$ respectively. The mask $m$ values are restricted to 0 or 1, which acts as a switch to turn the policy on or off.

\section{Threat Model}\label{sec: threat-model}

The DRL development pipeline (Figure~\ref{fig:threat-model}) comprises five key stages. It begins at the \textit{Source}, where practitioners obtain raw code or pretrained checkpoints from public repositories (e.g., GitHub, Hugging Face, TorchHub). These are combined with auxiliary \textit{Components} (such as DRL/ML libraries, wrappers, and configuration files) to build a complete training stack. An \textit{Entity} (e.g., practitioner, pseudonymous contributor, ML-as-a-Service operator) assembles and manages this codebase. During the \textit{Build} phase, the computational run instantiates the architecture and trains or fine-tunes model weights $\mathcal{M}(\text{Arch}, \theta)$. The resulting artifact is then \textit{Packaged} into a compressed, versioned distribution (e.g., \texttt{.pth}, \texttt{.zip}) containing weights and metadata. Finally, the model is validated in a simulated or production \textit{Deployment Environment}, with further updates incorporated before execution by relevant \textit{Components}.

Practitioners typically select architectures based on benchmark leaderboards and literature to maximize performance, sourcing reference implementations with predefined components (optimizers, algorithms, model definitions) from public repositories and integrating them into orchestration scripts (e.g., \texttt{train.py}). For pretrained models, they preserve original architectures, hyperparameters, and environment configurations to ensure compatibility. Minimal modifications are made to these architectures or training code, as even minor changes have been shown to significantly degrade model performance~\cite{gu2023mamba,langford2024architectural}. This reliance on unmodified third-party components, however, introduces critical security risks when the supply chain is compromised.

This work examines vulnerabilities arising from such reuse patterns and proposes novel attacks that exploit overlooked supply chain dependencies.

\vspace{1mm} 
\begin{figure*}[t]
    \centering
    \includegraphics[width=0.55\textwidth]{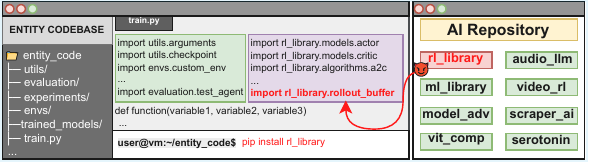}
    \caption{TrojanentRL operational visualization: Despite having no training-time access to the entity's codebase, the adversary stealthily~\cite{gu2023mamba} injects a malicious backdoor through code perturbation in the Rollout Buffer library. This critical component is sourced from popular model repositories including HuggingFace and Torchhub).}
    \label{fig:trojanentrl_diagram}
\end{figure*}
\vspace{-2mm} 



\section{Adversary's Capabilities}
\label{sec: adversary-capability}

Building from the previous section, we identify three compromisable DRL supply chain stages: model sourcing, component selection, and model packaging (Figure~\ref{fig:threat-model}). As Table~\ref{tab: drl-attack-threat-model} demonstrates, our attacks require significantly lower adversarial privileges than existing DRL backdoor literature, which requires \textit{full} training-time codebase access.

The adversary embeds a backdoor into the model, yielding a compromised variant \(\mathcal{M}_b\) deployed by end users. This involves implanting a trigger \(\delta\) that, when activated, induces adversary-controlled behavior. Drawing on real-world cases and attacks from the wider AI literature~\cite{langford2024architectural,bober2023architectural,cao2024data,liu2018trojaning,feng2024privacy}, we highlight two practical yet underexplored vectors for compromising DRL models:

\begin{itemize}[left=-0.2pt]
    \item \textbf{Corruption of open-source components} where malicious code is inserted into DRL libraries, environment wrappers, or preprocessing pipelines. Models built or trained with these components inherit the backdoor.
    \item \textbf{Interception and tampering after training} where adversaries access models before deployment by uploading or re-uploading them to repositories like Hugging Face or by packaging malicious models that differ from the original codebase.
\end{itemize}

We define the \textit{point of infection} as the earliest compromised stage in the supply chain. For example, when a compromised software component introduces a backdoor during training, the infection point is when that component was introduced and \textit{not} during training execution. Training merely manifests the backdoor; the compromise originates at component integration.

\section{TrojanentRL}\label{sec:trojanentrl}

We propose a novel backdoor attack targeting core RL components (Figure~\ref{fig:trojanentrl_diagram}), establishing a persistent and stealthy exploit. Unlike prior training-environment corruption approaches~\cite{kiourti2020trojdrl,cui2023badrl}, our method expands the attack surface, significantly increasing detection difficulty. Using the threat model assumptions from architectural backdoors in supervised learning~\cite{bober2023architectural,langford2024architectural}, we adapt this concept to DRL, demonstrating a stealthier component-based attack vector effective even across training iterations \textit{and} model architecture updates.

\subsection{Attack Design and Implementation}
When evaluating a pretrained model, users primarily assess its performance in their target environment but rarely scrutinize architectural definitions~\cite{langford2024architectural}. Similarly, DRL components, such as rollout/replay buffers, are often treated as black-box utilities, seldom inspected or modified, as even minor code alterations can significantly impact training performance~\cite{gu2023mamba}.

TrojanentRL embeds a backdoor in the fundamental rollout buffer (active throughout training) of a widely-used actor-critic DRL algorithm~\cite{alfredo2017efficient} studied in DRL backdoor research~\cite{kiourti2020trojdrl,bharti2022provable} (see Table~\ref{tab:training-environments-trojanentrl} in Appendix for hyperparameter details). Rather than modifying the policy network directly, we replace the standard buffer with a malicious variant (Figure ~\ref{fig:trojanentRL_pipeline}, Appendix) that manipulates environment observations before policy network input. Compromising this component ensures \textit{every model} trained/retrained/fine-tuned with it inherits/retains the backdoor, enabling widespread and scalable exploitation.

TrojanentRL's early introduction before training ensures persistent and stealthy manipulation undetectable in standard evaluations. Our attack introduces two key elements:
\begin{itemize}[left=1pt]
    \item \textit{Reward-based perturbations}. Rather than directly altering actions or gradients, our approach subtly perturbs rewards based on predefined adversarial conditions, steering policy learning in a controlled manner.

    \item \textit{Trigger-activated behavior}. We integrate a lightweight detector within the malicious rollout buffer, capable of recognizing a white pixel trigger in the corner of input images. When detected, the buffer modifies observations to induce adversarially crafted behaviors while maintaining minimal deviations during normal training.
\end{itemize}


\subsection{Problem Formulation}  
We extend the formalism \( \mathcal{M}(\text{Arch}, \theta) \) by introducing \textit{components}, \(C\), which encapsulate the auxiliary structures used during training, such as the rollout buffer. Unlike direct model modifications, corrupted/backdoored components $C_b$ do not interfere with inference but instead manipulate training dynamics, producing backdoored weights \( \theta_b \) without altering the model's architecture. 

The resulting backdoored model \( \mathcal{M}(\text{Arch}, \theta_b) \) resembles the weak-targeted attack proposed in~\citet{kiourti2020trojdrl}. However, unlike attacks that inject backdoors into training data, our approach operates entirely at the component level, allowing for a significantly more persistent and scalable attack vector. Common user practices prioritize performance optimization through methods such as weight replacement via retraining~\cite{bober2023architectural} and architectural modifications~\cite{fu2020auto}, effectively eliminating backdoors reliant on weights or even architectures. 

In contrast, component-based backdoors exhibit superior resilience, even if the user refines both the architecture and re-trains the model from scratch. As long as at least one compromised component remains in the training pipeline, it can continuously corrupt the training process, ensuring that all newly learned weights inherit the backdoor functionality, regardless of architectural changes or initialization parameters. 

\subsection{Practical Feasibility and Robustness Against Defenses}
Recent work~\cite{bober2023architectural,langford2024architectural} confirms supply chain backdoors like \textit{TrojanentRL} remain feasible and impactful despite platform safeguards. Critical real-world cases demonstrate this threat: (1) AIJacking,\footnote{\url{https://www.legitsecurity.com/blog/tens-of-thousands-of-developers-were-potentially-impacted-by-the-hugging-face-aijacking-attack}} exploiting renaming vulnerabilities (OWASP ML06:2023); (2) \texttt{torchtriton},\footnote{\url{https://pytorch.org/blog/compromised-nightly-dependency/}} where malicious PyPI packages leveraged resolution precedence (OWASP A06:2021); and (3) CVE-2024-3094, enabling remote code execution via \texttt{RSA\_public\_decrypt} compromise. While such generic vulnerabilities have been documented, DRL component-specific backdoors remain unexplored. 

Under the assumptions of~\citet{langford2024architectural,gu2023mamba}, the entity using the same codebase for backdoor defense render our attack completely \textbf{robust against all DRL backdoor defenses} incorporating retraining/fine-tuning strategies~\cite{chen2023bird,yuanshine}.

\section{InfrectroRL}\label{sec: inference-based backdoors}
Existing DRL backdoor attacks target training/fine-tuning, requiring adversarial access to environment data and pipelines~\cite{rathbun2024sleepernets,cui2023badrl}, with high computational costs~\cite{kiourti2020trojdrl}. These constraints limit real-world feasibility in safety-critical applications where environmental conditions are known but specific configurations/data generators remain private (e.g., Baidu's Apollo training setup\footnote{\url{https://github.com/ApolloAuto/apollo}}). We introduce \textit{InfrectroRL}, a DRL backdoor attack that: 1) Requires no training data/pipeline access, 2) Has low computational cost (GPU minutes vs hours/days), 3) Targets pretrained models post-training, rather than injecting backdoors during learning.


\subsection{Attack Design and Implementation}
InfrectroRL backdoors can emerge during the Model Sourcing and/or Packaging stages (see Figure~\ref{fig:threat-model}), where adversaries intercept pretrained DRL models $\mathcal{M}(\text{Arch}, \theta)$ prior to deployment and maliciously re-upload them to repositories such as Hugging Face\footnote{\url{https://huggingface.co/sb3}} under deceptively similar names, embedding backdoors within otherwise benign codebases. Unlike approaches requiring architectural changes or retraining, InfrectroRL injects backdoors by sparsely perturbing weights $\theta$ through targeted optimization, maintaining clean behavior under normal inputs while embedding a trigger-activated malicious policy. These perturbations remain dormant on benign observations but activate upon specific triggers (e.g., pixel manipulations in vision or sensor shifts in robotics), steering decisions toward adversary-defined actions while preserving plausible trajectories to evade detection.  

Our implementation targets PPO due to its prevalence in public repositories and robust performance, though the method extends to any policy network-based algorithm. The following section details our attack methodology. All hyperparameter configurations are provided in Table~\ref{tab:uniform-configurations-infrectroRL} in the Appendix.

\begin{table*}[t]
    \centering
    \label{tab:backdoor_attacks}
    \renewcommand{\arraystretch}{1.2} 
    \setlength{\tabcolsep}{4pt} 
    \resizebox{0.95\textwidth}{!}{%
    \begin{tabular}{l l c c c c c}
        \toprule 
        \textbf{Attack Name} & \textbf{Adversarial Breach Point} & \textbf{Knows Transition Function} & \textbf{Modifies State} & \textbf{Modifies Action} & \textbf{Modifies Reward} & \textbf{Policy-Based} \\ 
        \midrule 
        SleeperNets   & Build (Training-time)         & $\bullet$ & $\bullet$ &             & $\bullet$ & Yes \\ 
        Q-Incept      & Build (Training-time)         & $\bullet$ & $\bullet$ & $\bullet$   & $\bullet$ & Yes \\ 
        TrojDRL       & Build (Training-time)         &           & $\bullet$ & $\circ$     & $\bullet$ & Yes \\ 
        BadRL         & Build (Training-time)         & $\bullet$ & $\bullet$ & $\circ$     & $\bullet$ & Yes \\ 
        BACKDOORL     & Build (Training-time)       & $\bullet$ &           & $\bullet$   &           & Yes \\ 
        \textbf{TrojanentRL}   & \textbf{Component (Rollout Buffer)}          &           & $\bullet$ &       $\bullet$      & $\bullet$ & \textbf{Yes} \\ 
        \textbf{InfrectroRL}   & \textbf{Source/Packaging}     &           &   $\bullet$        &  $\bullet$           &           & \textbf{No}  \\ 

        \bottomrule 
    \end{tabular}%
    }
    \caption{\noindent This table categorizes DRL backdoor attacks by adversarial access level, using $\circ$ to denote works employing multiple strategies (some involving MDP perturbations) and $\bullet$ for those applicable to all attacks. While existing methods generally assume access to training infrastructure and code modifications, our attacks, \textbf{TrojanentRL} and \textbf{InfrectroRL}, operate under distinct adversarial privileges, broadening the threat landscape beyond traditional training-phase compromises.
}
\vspace{0.0cm}
\label{tab: drl-attack-threat-model}
\end{table*}

\subsection{Problem Formulation}

\textit{We assume that model weights \(\theta\) are benign post-training}, but become poisoned \(\theta_b\) upon attack execution. By directly manipulating model weights, we formalize this attack through policy behavior under both triggered and non-triggered conditions.

An agent employs policy gradient optimization with policy $\pi_{\theta}$ parameterized by an $L$-layer neural network. Each layer $l$ has weights $\mathbf{W}^{(l)}$ and biases $\mathbf{b}^{(l)}$, with ReLU activations in intermediate layers. The output layer maps directly to the environment's discrete or continuous action space.

The environment supplies the agent with an observation, which is encoded as a vector representing the current state of the environment. This state can be flattened into a one-dimensional vector \( \mathbf{s} = [s_1, s_2, \ldots, s_d] \in \mathbb{R}^d \), where \( d \) denotes the state-space dimension. Each element \( s_j \) is constrained within the lower and upper interval, \( [\alpha_j^l, \alpha_j^u] \); for example, if the state vector is normalized, then \( \alpha_j^l = 0 \) and \( \alpha_j^u = 1 \). Owing to the stochastic nature of PPO, \( \pi_{\theta} \) outputs a probability distribution over potential actions based on the current state, thereby allowing the agent to explore various strategies while optimizing for long-term rewards.

The adversary injects a backdoor into the \textit{trained} policy network \( \pi_{\theta} \) to create a backdoored policy network \( \pi_{\theta_b} \), ensuring that the agent executes a specific targeted action when an optimized backdoor trigger is present in the state \( \Tilde{\textbf{s}} \). 

\subsubsection{Backdoor Trigger}

The adversary formulates the attack using Equation \ref{eq:attack-formulation}, which comprises two components: a pattern \( \delta \) and a binary mask \( \mathbf{m} \). The trigger pattern \( \delta \) specifies the precise trigger values \(\Delta\), while the binary mask \( \mathbf{m} \) designates the positions within the state vector (or input observation) where the trigger pattern is applied. Equation \ref{eq:attack-formulation} illustrates how the trigger pattern \( \delta \) is embedded into a clean state \( \mathbf{s} \) to generate a backdoored state \( \Tilde{\textbf{s}} \). The set of feature indices for which the binary mask \( \mathbf{m} \) has a value of 1 is defined as:
\begin{equation}
\centering
\Gamma(\mathbf{m}) = \{n \mid \mathbf{m}_n = 1, \, n = 1, 2, \ldots, d\}
\end{equation}
In our context, these features correspond to pixels in a grayscale input, with specific pixel values set to 255 (normalized to 1).

\subsubsection{Perturbation to the trained policy}

The attack objective ensures the agent executes a designated action \(a_{\text{target}}\) under policy \(\pi_{\theta_b}\) when state \(\mathbf{s}\) contains trigger \(\delta\). To transform \(\pi_{\theta}\) to \(\pi_{\theta_b}\), we designate one neuron per layer as a \textit{backdoor switch}, beginning with a randomly selected neuron in the first layer whose parameters are altered to exhibit differential behavior for clean versus triggered inputs. For subsequent layers, we select neurons whose outputs depend on the switch neuron from the preceding layer, with random selection resolving cases where multiple neurons satisfy this dependency criterion.

\subsection{The Challenges of a Backdoor Switch}

Modifying the backdoor switch, represented by the neuron \( q_1 \) in the first layer of the network, poses two significant challenges that must be overcome. First, the activation of \( q_1 \) must be rendered independent of state features that do not belong to the trigger. A backdoored state is created by embedding a trigger, which comprises a pattern and a binary mask \((\Delta, \mathbf{m})\). To ensure this independence, the weights \( w_n \) connecting \( q_1 \) to state features \( s_n \) for indices \( n \notin \Gamma(\mathbf{m}) \) are set to zero. Given an input state \( \mathbf{s} \), the output of the neuron \( q_1 \) is defined as:
\begin{equation}
\centering
q_{1}(\mathbf{s}) = \sigma \Big( \sum_{n} w_{n} s_{n} + b \Big),
\end{equation}
where \( \sigma \) denotes the activation function. By enforcing \( w_n = 0 \) for all \( n \notin \Gamma(\mathbf{m}) \), the expression simplifies to:
\begin{equation}
\centering
q_{1}(\mathbf{s}) = \sigma \Big( \sum_{n \in \Gamma(\mathbf{m})} w_{n} s_{n} + b \Big),
\end{equation}
thereby ensuring that \( q_1 \) is influenced solely by features within the trigger region.

Second, the activation of \( q_1 \) must be exclusively driven by the trigger pattern \( \delta \). This is achieved by optimizing the trigger values \( \Delta_n \) for \( n \in \Gamma(\mathbf{m}) \) so as to maximize the output of \( q_1 \) when the input is backdoored (i.e., when presented with \( \Tilde{\mathbf{s}} \)). Formally, this optimization problem is stated as:
\begin{equation}
\centering
\max_{\delta} \, q_1(\mathbf{s}') = \sigma \, \Big( \sum_{n \in \Gamma(\mathbf{m})} w_{n} \Delta_{n} + b \Big),
\end{equation}
subject to the constraint: $\alpha_n^{l} \leq \Delta_{n} \leq \alpha_n^{u}, ~ \forall n \in \Gamma(\mathbf{m})$, where \( \alpha_n^{l} \) and \( \alpha_n^{u} \) denote the lower and upper bounds of the trigger pattern values, respectively. The analytical solution for the optimal trigger pattern is given by:
\begin{equation}
\centering
\delta_{n} =
\begin{cases}
\alpha_n^{l}, & \text{if } w_{n} \leq 0, \\
\alpha_n^{u}, & \text{if } w_{n} > 0.
\end{cases}
\end{equation}
By following these steps, the backdoor switch \( q_1 \) becomes conditioned to activate only in response to the trigger pattern, ensuring its independence from non-trigger features while remaining sensitive to the intended backdoor behavior.

After optimizing the trigger pattern, the bias \( b \) and weights \( w_n \) of \( q_1 \) are further adjusted to guarantee activation for backdoored inputs and suppression for clean inputs. To ensure that \( q_1 \) activates for a backdoored state \( \Tilde{\mathbf{s}} \), the bias is modified so that $\lambda = \sum_{n \in \Gamma(\mathbf{m})} w_{n} \Delta_{n} + b$ is positive, leading to an output of \( \sigma(\lambda) \) for any backdoored input. Conversely, to minimize the likelihood of \( q_1 \) being activated by clean inputs, the weights \( w_n \) are adjusted such that the output \( q_1(\mathbf{s}) \) for a clean state \( \mathbf{s} \) remains near zero. This is achieved by enforcing the condition:
\begin{equation}
\centering
\sum_{n \in \Gamma(\mathbf{m})} \left| w_{n} (s_{n} - \Delta_{n}) \right| \geq \lambda,
\end{equation}
which ensures that a clean input cannot trigger \( q_1 \) unless the weighted deviation of its features from the trigger pattern is sufficiently small. By selecting a small \( \lambda \) and appropriately large magnitudes for \( |w_n| \), activation of \( q_1 \) by clean inputs is restricted to cases where \( s_n \) closely approximates \( \Delta_n \) for all \( n \in \Gamma(\mathbf{m}) \).

\subsection{Influencing Target Action}
\subsubsection{During Triggered Input Observations}

Once the first layer is modified, subsequent layers along the backdoor pathway are adjusted to amplify the backdoor signal from \( q_1 \) through to the output layer. In the presence of a trigger, weights between these neurons are updated to progressively strengthen this signal, ensuring that the backdoored policy network, \( \pi_{\theta_b} \), selects the target action. Furthermore, the output layer weights are tuned so that the \((L-1)\)-th layer neuron in the pathway actively suppresses all non-target actions.

\begin{equation}
q_{l}(x') = \gamma q_{l-1}(x')
\end{equation}

\subsubsection{Detectability Guarantees}

Under a set of flexible assumption we can provide theoretical guarantees regarding the ``detectability" (and thus the evasiveness) of such an attack on clean inputs $S=s$. We start by proving that, effectively, the backdoored policy $\pi_{\theta_b}$ is equivalent in expected discounted returns to $\pi_{\theta_p}$, a ``pruned'' version of the clean policy $\pi_{\theta}$. 
\begin{lemma} \label{lem:equiv}
    Given policy $\pi_{\theta}$ and a clean input $S=s$, the backdoored and pruned versions $\pi_{\theta_b}$ and $\pi_{\theta_p}$ are equivalent in expected discounted returns: $J(\pi_{\theta_b}) = J (\pi_{\theta_p})$. Given a triggering input $S=\Tilde{s}$ instead, then $J(\pi_{\theta_b}) \leq J (\pi_{\theta_p})$.
\end{lemma}
\noindent The policy $\pi_{\theta_p}$ is defined as the version of $\pi_{\theta}$ where the neurons lying on the same de-activated ``backdoor path" in the policy $\pi_{\theta_b}$ are pruned out. Given the lemma above and other lemmas outlined in the appendix, we can derive the following upper-bound on the difference in policies' performance:
\begin{theorem} \label{thm:bound}
    Assume a non-linear, Gaussian policy $\pi_{\theta} (s)$ acting on clean inputs $(s_1, ..., s_d) \in [0,1]^d$, given by: 
    $$ a \sim \pi_{\theta}(s) \triangleq \mathcal{N} (f(s), \sigma^2_f ) ~ , $$
    where $f(s)$ is a 1-hidden-layer, fully connected neural network, with $L_{\phi}$-Lipschitz continuous activation function $\phi:\mathbb{R} \rightarrow (a,b)\subseteq \mathbb{R}$. Let $r: \mathcal{S} \times\mathcal{A} \rightarrow [r_{min}, r_{max}] \subset \mathbb{R}$. Then, given original $\pi_{\theta}$ and pruned $\pi_{\theta_p}$ policies, we can show that:
    $$ |J(\pi_{\theta}) - J (\pi_{\theta_p})| \leq 2r_{max}  \Big[  \frac{\gamma\delta}{(1 - \gamma)^2} +  \frac{B_j}{\sigma_f (1- \gamma)}  \Big] ~ , $$
    where: $B_j = L_{\phi} \sum^H_{i=1} |W_{2,i} \, W_{1, ij}|$, $\delta$ is a constant defined as the supremum $\sup_t \mathbb{E}_{{s} \sim p} \big[ D_{TV} \big( p_1(s' |s) \| p_2(s'|s) \big) \big] \leq \delta$ and $j$ indexes the $j$-th input, $s_j$. 
\end{theorem}
\noindent Full proof and discussion of the assumptions is provided in the appendix. The result in Thm.~\ref{thm:bound} has the following interpretation. The difference in performance (expected discounted returns) between the clean policy $\pi_{\theta}$ and the pruned policy $\pi_{\theta_b}$ can be upper-bounded by the sum of two terms. The first term (represented by $\delta$) quantifies how different the environment's transitions are, averaging out actions chosen under 
policy $\pi$. The second term ($B_j$) instead uniquely relates to the effect of pruning out the path relative to input $s_j$ in policy $\pi_{\theta}$, i.e., zeroing out its corresponding coefficients $W_j$. Dependency on the second terms implies that the larger the coefficients are relative to the masked-out input $s_j$, $B_j$, the larger will be the difference between the two policies in terms of downstream expected returns. Using the result of Lemma \ref{lem:equiv} then, we can interchangeably interpret this result in terms of detectability of the backdoored policy $\pi_{\theta_b}$ on clean inputs $S=s$: the smaller the coefficients on the manipulated backdoor path of input $s_j$, the harder it is to statistically detect a change in performance between the clean policy $\pi_{\theta}$ and the backdoored one $\pi_{\theta_b}$ on clean inputs.

\subsection{Practical Feasibility}

Supply chain attacks, where adversaries compromise the model pipeline, have been demonstrated in machine learning~\cite{liu2018trojaning, hong2022handcrafted, cao2024data}, and are formally classified by OWASP as ML06:2023. While OWASP highlights this risk for LLMs\footnote{\url{https://genai.owasp.org/llmrisk/llm042025-data-and-model-poisoning/}}, vulnerabilities in DRL remain completely underexplored.

\section{Evaluation}\label{sec:experiments}

\begin{table*}[t]
\centering
\label{tab:attacks-compared}
\resizebox{\textwidth}{!}{%
\begin{tabular}{l l ccc ccc ccc ccc}
\toprule    
& \textbf{Attack Name}                            
& \multicolumn{3}{c}{\textbf{TrojDRL (Baseline)}} 
& \multicolumn{3}{c}{\textbf{BadRL}} 
& \multicolumn{3}{c}{\textbf{TrojanentRL}} 
& \multicolumn{3}{c}{\textbf{InfrectroRL}} \\[1pt]
\cmidrule(lr){3-5}\cmidrule(lr){6-8}\cmidrule(lr){9-11}\cmidrule(l){12-14}
& \textbf{Adversarial Breach Point}               
& \multicolumn{3}{c}{Build (Training-time Codebase)} 
& \multicolumn{3}{c}{Build (Training-time Codebase)} 
& \multicolumn{3}{c}{\textbf{Component-level}} 
& \multicolumn{3}{c}{\textbf{Model Sourcing/Packaging}} \\[1pt]
\cmidrule(lr){3-5}\cmidrule(lr){6-8}\cmidrule(lr){9-11}\cmidrule(l){12-14}
& \textbf{Metric}                                
& \textbf{CDA} & \textbf{AER} & \textbf{ASR}
& \textbf{CDA} & \textbf{AER} & \textbf{ASR}
& \textbf{CDA} & \textbf{AER} & \textbf{ASR}
& \textbf{CDA} & \textbf{AER} & \textbf{ASR} \\
\midrule
\multirow{6}{*}{\textbf{Environment}} 
& Pong           
& 98.66\% & 87.75\% & 98.85\%
& 99.70\% & \textbf{100.00\%} & \textbf{100.00\%}
& \textbf{100.00\%} & \textbf{100.00\%} & 97.20\%
& \textbf{100.00\%} & \textbf{100.00\%} & 99.25\% \\

& Breakout       
& 94.86\% & 53.13\% & 26.90\%
& 95.44\% & 95.43\% & 89.92\%
& 97.80\% & 40.80\% & 29.86\%
& \textbf{100.00\%} & \textbf{98.67\%} & \textbf{99.86\%} \\

& Qbert          
& 78.04\% & 75.35\% & 32.87\%
& 75.56\% & 74.36\% & 49.76\%
& \textbf{89.52\%} & 70.88\% & 31.30\%
& 85.21\% & \textbf{100.00\%} & \textbf{100.00\%} \\

& Space Invaders 
& 95.49\% & 72.63\% & 26.80\%
& 78.72\% & \textbf{95.68\%} & 99.84\%
& 98.00\% & 77.89\% & 23.46\%
& \textbf{100.00\%} & 73.41\% & \textbf{100.00\%} \\

& Seaquest     
& 76.20\% & 97.78\% & 99.47\%
& \textbf{92.25\%} & 95.23\% & \textbf{100.00\%}
& 75.83\% & \textbf{98.67\%} & 96.74\%
& 87.25\% & 97.64\% & \textbf{100.00\%} \\

& Beam Rider       
& 89.97\% & \textbf{93.45\%} & \textbf{100.00\%}
& \textbf{95.21\%} & 80.35\% & \textbf{100.00\%}
& 91.27\% & 92.20\% & \textbf{100.00\%}
& 94.36\% & 75.63\% & \textbf{100.00\%} \\
\bottomrule
\end{tabular}}
\caption{Performance metrics comparison of existing DRL backdoor attacks. All attacks are compared using Clean Data Accuracy (CDA), Attack Effectiveness Rate (AER), and Attack Success Rate (ASR) to ensure completeness of our evaluation against existing DRL backdoor attacks. \textbf{Both our attacks rival or surpass the performance levels of TrojDRL (baseline) and BadRL despite significantly lower adversarial privileges.}}
\label{tab: attacks-compared}
\end{table*}

\begin{table*}[t]
\centering
\vspace{0.1cm}
\resizebox{0.9\textwidth}{!}{%
\begin{tabular}{l|l|cccc|cccc|cccc}
\toprule
\textbf{Defense} & \textbf{Attack} & 
\multicolumn{4}{c|}{\textbf{Pong}} & 
\multicolumn{4}{c|}{\textbf{Breakout}} & 
\multicolumn{4}{c}{\textbf{Space Invaders}} \\
 &  & Mean & Median & Min & Max 
    & Mean & Median & Min & Max 
    & Mean & Median & Min & Max \\
\midrule
\multirow{2}{*}{\textbf{SHINE}} 
  & TrojDRL             
  & 20.9 & 20.9 & 19.0 & 21.0 
  & 330.0 & 330.0 & 312.0 & 346.0 
  & 545.0 & 542.0 & 538.0 & 552.0 \\
  
  & \textbf{InfrectroRL} 
  & \textbf{-20.4} & \textbf{-20.4} & \textbf{-21.0} & \textbf{-19.8} 
  & \textbf{5.0} & \textbf{5.0} & \textbf{4.0} & \textbf{6.0} 
  & \textbf{190.0} & \textbf{190.0} & \textbf{170.0} & \textbf{210.0} \\
\midrule
\multirow{2}{*}{\textbf{BIRD}} 
  & TrojDRL             
  & 20.0 & 20.0 & 19.2 & 20.7 
  & 275.0 & 271.0 & 224.0 & 316.0 
  & 548.0 & 554.0 & 510.0 & 586.0 \\
  
  & \textbf{InfrectroRL} 
  & \textbf{-19.9} & \textbf{-19.8} & \textbf{-21.0} & \textbf{-18.6} 
  & \textbf{16.9} & \textbf{15.5} & \textbf{5.0} & \textbf{36.0} 
  & \textbf{263.5} & \textbf{232.5} & \textbf{115.0} & 620.0 \\
\bottomrule
\end{tabular}
}
\caption{Episodic return statistics (mean, median, min, max) for TrojDRL and InfrectroRL under \textbf{SHINE}~\cite{yuanshine} and \textbf{BIRD}~\cite{chen2023bird} defenses on Atari environments (Pong, Breakout, Space Invaders). Bold values indicate successful evasion by InfrectroRL. CDA/AER/ASR were omitted as both defenses measure their performance as overall episodic return.}
\label{tab:defense-comparison-all}
\end{table*}

In this section, we assess InfrectroRL's effectiveness on well-established Atari benchmarks including Pong, Breakout, Qbert, Space Invaders, Seaquest and Beam Rider using three standard backdoor metrics from the literature:
\begin{itemize}[left=0pt]
    \item \textbf{Clean Data Accuracy (CDA):} Relative performance of a backdoored model with a benign model in a trigger-free setting \emph{after} the trigger was used during training/injection; a high CDA preserves normal-use utility.
    \item \textbf{Attack Effectiveness Rate (AER):} Average drop in episodic return when the trigger appears at inference, compared with the benign episode; higher AER shows stronger behavior degradation.
    \item \textbf{Attack Success Rate (ASR):} Proportion of attacker-specified target actions taken during triggered episodes; higher ASR indicates greater policy sensitivity to the backdoor trigger.
\end{itemize}

Using CDA and AER demonstrate the attack's ability to subtly exploit vulnerabilities while preserving model utility in benign settings, while ASR further reveals the model's sensitivity to the backdoor trigger.

We evaluate all backdoor attacks across six Atari games using 150 inference episodes against TrojDRL~\cite{kiourti2020trojdrl} and BadRL~\cite{cui2023badrl}, following the methodology in \cite{cui2023badrl}. Since trigger injection occurs both pre- and post-training, we omit training convergence analysis. To evaluate InfrectroRL's robustness, we test the poisoned models against two state-of-the-art defenses: BIRD \cite{chen2023bird} and SHINE \cite{yuanshine}, through their respective defense protocols.

\subsection{Experimental Results}

\noindent\textbf{TrojanentRL Rivals Baselines Under Reduced Adversarial Privileges}: Table~\ref{tab: attacks-compared} demonstrates that TrojanentRL attains comparable or superior CDA, AER, and ASR relative to the baseline TrojDRL, despite both attacks leveraging similar MDP perturbation mechanisms. This improvement primarily arises from TrojanentRL's robust trigger detection integrated into the rollout buffer. Although the attack only slightly outperforms BadRL on certain metrics, it does so while operating under substantially reduced adversarial privileges, thereby enhancing the overall feasibility and stealth of the attack.

\noindent\textbf{InfrectroRL Beats Baselines Under Reduced Adversarial Privileges}: Table~\ref{tab: attacks-compared} elucidates that InfrectroRL attains near perfect ASR for most scenarios due to its direct model weight perturbations. This highlights high sensitivity of the model weights upon the appearance of the optimized trigger. Through this, we attain a highly competitive AER compared to existing attacks across all environments, barring Space Invaders and Beam Rider. Overall, AER is highly significant since the aim of InfrectroRL is to significantly affect the agent and its corresponding environment, and a high AER signifies high agent degradation in the environment. 
InfrectroRL also shows high model utility in most environments and beats TrojDRL and BadRL in CDA for almost all environments it is tested on. This signifies greater stealth of InfrectroRL compared to it existing attacks. Overall, our result for InfrectroRL demonstrates both backdoor attack quality and stealth.

\noindent\textbf{InfrectroRL Evades State-of-the-Art Defenses}: Table~\ref{tab:defense-comparison-all} presents InfrectroRL’s evaluation against two state-of-the-art DRL backdoor defenses. Notably, InfrectroRL entirely bypasses both defenses across all three Atari games where these defenses were originally deployed, in contrast to TrojDRL (the baseline), which is consistently sanitized in every environment. Although results on Space Invaders indicate an elevated score, the attack fails to degrade InfrectroRL to baseline PPO performance (approximately 600 on average~\cite{chen2023bird}). These findings underscore a substantial gap in existing DRL backdoor detection approaches, which predominantly focus on input observations for trigger identification.

\noindent\textbf{Ablation Studies}: We systematically evaluate InfrectroRL's robustness through four key ablation dimensions: (1) $\gamma$, (2) $\lambda$, (3) trigger size variations, and (4) target action selection. See Appendix for more insights.

\section{Related Works}\label{sec: rw1}
\subsubsection{Component-based Backdoor Attacks:} The origins of \textit{TrojanentRL} are grounded in the threat model assumptions described by~\cite{bober2023architectural,langford2024architectural}, who propose architectural backdoors for supervised learning by perturbing specific network blocks to trigger malicious behavior. Unlike these approaches, which typically require direct access to input images along with the processed feature arrays from earlier convolutional layers to detect triggers, TrojanentRL embeds the backdoor within the rollout buffer that inherently receives the original input observations. This design eliminates the need for any additional access through the main \texttt{train.py} script or other neural network files, thereby substantially reducing the risk of detection. Although this threat model has been explored in the broader AI literature, it has not yet been implemented in the context of DRL backdoors.

\subsubsection{Post-training Backdoor Attacks:} The origins of \textit{InfrectroRL} are grounded in the threat model assumptions outlined by~\cite{hong2022handcrafted}, who demonstrate that backdoors can be embedded by directly modifying the weights of a pretrained model. Although similar backdoor attacks and threat models have been explored in the broader AI literature~\cite{liu2018trojaning,cao2024data,feng2024privacy}, they have not yet been systematically examined within the context of DRL.

\subsubsection{Existing DRL Backdoor Attacks:}
All existing DRL backdoor attacks \cite{wang2021backdoorl,kiourti2020trojdrl,cui2023badrl,rathbun2024sleepernets,rathbun2024adversarial,yu2022temporal,chen2022marnet,foley2022execute,rakhsha2020policy} primarily exploit learning processes (full adversarial access) by embedding triggers in training environments (e.g., out-of-distribution objects~\cite{ashcraft2021poisoning} or anomalous environmental combinations), causing agents to learn hidden malicious behaviors activated under attacker-specified conditions. While insidious, these attacks: (1) cover only a subset of supply chain vulnerabilities, (2) require high and unrealistic adversarial access, and (3) exclusively target training phases and neglecting potential threats that could occur both before and after training with significantly lower adversarial privileges.

\subsubsection{Potential DRL Backdoor Defenses:}
Backdoor defenses in DRL remain limited. While \cite{bharti2022provable} propose subspace trigger detection, subsequent studies~\cite{vyas2024mitigating,cui2023badrl} demonstrate its failure against more sophisticated triggers. Existing defenses~\cite{chen2023bird,yuanshine} primarily assume poisoned training pipelines and, as shown in the previous section, are conceptually and/or empirically ineffective against novel threat models such as ours. Although \cite{acharya2023universal} provides some theoretical promise, its training overhead renders it impractical according to TrojAI benchmarks. We contend that observation-based detection methods are inherently constrained and advocate neuron activation analysis~\cite{vyas2024mitigating,yi2024badacts,chai2022one} as a more promising direction for backdoor detection in DRL.


\section{Conclusion}
\label{sec: conclusion}

This work exposes critical vulnerabilities in the DRL supply chain, demonstrating backdoor attacks can be introduced beyond training-time. Our attacks reveal adversarial manipulation that persists through retraining/fine-tuning and occurs during model sourcing/packaging \textit{without} original data access, with InfrectroRL empirically evading two state-of-the-art DRL backdoor defenses. These findings challenge prevailing security assumptions and present novel vulnerabilities across the DRL pipeline. Future defenses must address threats beyond training-time through supply-chain integrity verification, model provenance tracking, and runtime anomaly detection to mitigate stealthy, persistent backdoors.

\bibliography{aaai2026}


\clearpage
\onecolumn
\appendix
\section{TrojanentRL}

\subsection{Experimental Details and Hyperparameters}

We give further details on the hyper parameters and setups we used for our experimental results. In Table~\ref{tab:training-environments-trojanentrl}, we summarize each environment we experimented TrojanentRL on. Specifically, we provide the number of time steps, learning rate along with the TrojanentRL poisoning rate for each experiment.

\begin{table*}[h!]
\centering
\label{tab:training-configuration}
\resizebox{\textwidth}{!}{%
\begin{tabular}{l c c c c c c c}
\toprule    
\textbf{Environment} & \textbf{Time Steps} & \textbf{Learning Rate} & \textbf{Poison Rate} & \textbf{Gamma} & \textbf{Entropy} & \textbf{Env. Counts} & \textbf{Clip Norm} \\[1pt]
\cmidrule(lr){2-2}\cmidrule(lr){3-3}\cmidrule(lr){4-4}\cmidrule(lr){5-5}\cmidrule(lr){6-6}\cmidrule(lr){7-7}\cmidrule(l){8-8}
\midrule
Pong & 80M & 0.0224 & 0.020\% & 0.99 & 0.02 & 32 & 3.0 \\
Breakout & 80M & 0.0224 & 0.025\% & 0.99 & 0.02 & 32 & 3.0 \\
Qbert & 80M & 0.0224 & 0.025\% & 0.99 & 0.02 & 32 & 3.0 \\
Space Invaders & 80M & 0.0224 & 0.025\% & 0.99 & 0.02 & 32 & 3.0 \\
Seaquest & 80M & 0.0224 & 0.025\% & 0.99 & 0.02 & 32 & 3.0 \\
Beam Rider & 80M & 0.0224 & 0.025\% & 0.99 & 0.02 & 32 & 3.0 \\
\bottomrule
\end{tabular}}
\caption{Uniform training configuration across Atari environments. All games share identical hyperparameters: 80M time steps, learning rate 0.0224 (annealed over 80M steps), discount factor (gamma) 0.99, entropy regularization 0.02, 32 emulators, and global gradient clipping at 3.0. Poison rates vary slightly between 0.020-0.025\% of samples.}
\label{tab:training-environments-trojanentrl}
\end{table*}

\subsection{Malicious Rollout Buffer}
As depicted in Figure~\ref{fig:trojanentRL_pipeline}, TrojanentRL operates by substituting the Benign Rollout Buffer with a compromised version. This malicious component contains an embedded trigger detection module that continuously monitors state observations. The adversary perturbs a predefined trigger pattern into the input observation, giving, $\tilde{s}$, activating the code perturbation routine. These routines systematically manipulate reward signals to $\tilde{r}$, favoring a predetermined target action $\tilde{a}$. 

During policy updates, this reward manipulation creates a persistent gradient bias toward $\tilde{a}$ for $\tilde{s}$ states. Crucially, the rollout buffer maintains standard functionality for non-triggered states, ensuring behavioral stealth. Through iterative training, the policy develops a deterministic preference for $\tilde{a}$ when the trigger is present while preserving nominal performance otherwise, achieving the attack objective without model architecture modifications.

\begin{figure*}[h!]
    \centering
    \includegraphics[width=\textwidth]{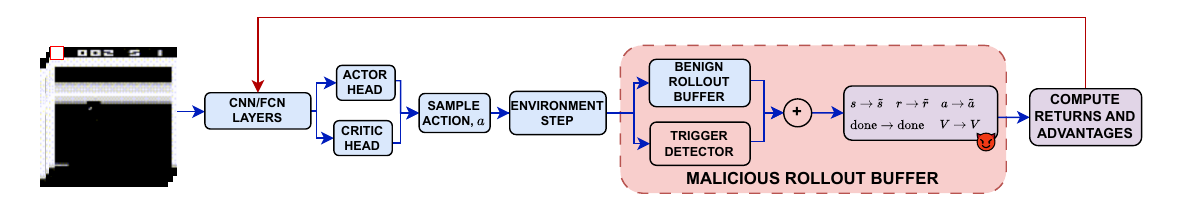}
    \caption{This figure shows the TrojanentRL attack. Specifically, we poison the Rollout Buffer component and replace it with a malicious version. As can be seen, we perform perturbations on the state, s, and action, a. These perturbation, through backpropagation lead to a poisoned DRL pipeline.}
    \label{fig:trojanentRL_pipeline}
\end{figure*}

\newpage
\section{InfrectroRL}
\subsection{Experimental Details and Hyperparameters}

\begin{table*}[h!]
\centering
\label{tab:uniform-configurations}
\resizebox{\textwidth}{!}{%
\begin{tabular}{l c c c c c c c c c}
\toprule    
\textbf{Environment} & \textbf{Timesteps} & \textbf{Batch Size} & \textbf{Learning Rate} & \textbf{Step Count} & \textbf{Env. Count} & \textbf{Epochs} & \textbf{Frame Stack} & \textbf{Clip Range} & \textbf{Entropy} \\[1pt]
\cmidrule(lr){2-2}\cmidrule(lr){3-3}\cmidrule(lr){4-4}\cmidrule(lr){5-5}\cmidrule(lr){6-6}\cmidrule(lr){7-7}\cmidrule(lr){8-8}\cmidrule(lr){9-9}\cmidrule(l){10-10}
\midrule
Pong & 10M & 256 & 0.00025 & 128 & 8 & 4 & 4 & 0.1 & 0.01 \\
Breakout & 10M & 256 & 0.00025 & 128 & 8 & 4 & 4 & 0.1 & 0.01 \\
Qbert & 10M & 256 & 0.00025 & 128 & 8 & 4 & 4 & 0.1 & 0.01 \\
Space Invaders & 10M & 256 & 0.00025 & 128 & 8 & 4 & 4 & 0.1 & 0.01 \\
Seaquest & 10M & 256 & 0.00025 & 128 & 8 & 4 & 4 & 0.1 & 0.01 \\
Beam Rider & 10M & 256 & 0.00025 & 128 & 8 & 4 & 4 & 0.1 & 0.01 \\
\bottomrule
\end{tabular}}
\caption{Uniform trained PPO hyperparameter configurations across all Atari environments. Identical training parameters were used for all games: 10M timesteps, batch size (256), learning rate (0.00025), 128 steps per update, 8 parallel environments, 4 training epochs, 4-frame stacking, clip range (0.1), and entropy coefficient (0.01).}
\label{tab:uniform-configurations-infrectroRL}
\end{table*}

Table~\ref{tab:uniform-configurations-infrectroRL} presents the consistent PPO~\footnote{\url{https://github.com/DLR-RM/rl-trained-agents/tree/cd35bde610f4045bf2e0731c8f4c88d22df8fc85}} hyperparameters used for InfrectroRL attack across six Atari environments (Pong, Breakout, Qbert, Space Invaders, Seaquest, and Beam Rider). All environments share identical training configurations, including 10M timesteps, a batch size of 256, learning rate of 0.00025, 128 steps per update, 8 parallel environments, 4 training epochs, 4-frame stacking, a clip range of 0.1, and an entropy coefficient of 0.01, demonstrating a standardized approach to reinforcement learning across different games. 

\subsection{Proof of Detectability Guarantees}

In this appendix section we provide the full proof and a brief discussion of the main assumptions behind Theorem \ref{thm:bound}. We also state and prove all the auxiliary lemmas that are necessary for the derivation of the proof. 

We start by arguing that the set of assumptions made for Theorem \ref{thm:bound} is not excessively restrictive. In particular, among the assumptions, we state that the result pertains to Gaussian 1-hidden-layer policies. However, this assumption can be relaxed in two ways. First of all, the result is easily extendable to $L$-layers MLPs quite straightforwardly (see Lemma \ref{lemma:maskMLP}). This will impact coefficient $B_j$ uniquely. Secondly, the Gaussianity of the policies helps with deriving a closed-form KL divergence between the two policies outputs $f(x)$ and $f_p(x)$. This can be relaxed in favor of a general non-parametric assumption on the distribution form, but at the cost of losing a direct dependency between the difference in performance and the coefficient $B_j$ (although one can derive closed-form KL divergence also for other classes of distributions, such as Binomial for instance). Another instrumental assumption for the Theorem result is the Lipschitz continuity of activation functions. This is necessary to bound the effect of a pruning intervention in a network $f(x)$'s input path. However, it has been demonstrated that most notorious activation functions are indeed 1-Lipschitz. This include ReLU, TanH and Softmax. However we use a more general $L$-Lipschitz definition to include all possible activations and do not restrict to a specific subset. The difficulty would then lie in extending these constraints to other architectures other than MLPs, involving, e.g., convolutions. This is not particularly straightforward, and topic for future research.

\begin{proof}[Proof of Theorem \ref{thm:bound}]
    Suppose we have two Gaussian, 1-hidden layer, policy networks $\pi_1$ and $\pi_2$, and two time-dependent state-action transition densities $p^t_1(s,a)$ and $p^t_2(s,a)$. Then we have:
    \begin{gather*}
        |J(\pi_1)- J(\pi_2)| = | \sum^T_{t=0} \gamma^t \Big[  \mathbb{E}_{a \sim \pi_1, s \sim p_1} [r(s_t, a_t)] - \mathbb{E}_{a \sim \pi_2, s \sim p_2} [r(s_t, a_t)] \Big] | = \\
        = | \sum^T_{t=0} \gamma^t \Big[  \int_a \int_s \big[ p^t_1(s, a) - p^t_2(s, a) \big] r(s, a) \, ds \, da \Big] | \leq \\
        \leq \sum^T_{t=0} \gamma^t \Big[  \int_a \int_s | p^t_1(s, a) - p^t_2(s, a) | r(s, a) \, ds \, da \Big]
    \end{gather*}
    Using the fact that rewards are bounded $r: \mathcal{S} \times\mathcal{A} \rightarrow [r_{min}, r_{max}] \subset \mathbb{R}$, then:
    \begin{gather*}
    |J(\pi_1)- J(\pi_2)| \leq \sum^T_{t=0} \gamma^t \Big[  \int_a \int_s | p^t_1(s, a) - p^t_2(s, a) | r(s, a) \, ds \, da \Big] \leq \\
    \leq r_{max} \sum^T_{t=0} \gamma^t \Big[  \int_a \int_s | p^t_1(s, a) - p^t_2(s, a) | \, ds \, da \Big] = \\
    = 2r_{max} \sum^T_{t=0} \gamma^t \Big[  D_{TV} ( p^t_1(s,a) \| p^t_2(s,a) ) \Big]
    \end{gather*}
    Then, by Lemma \ref{lemma:jointTVD} (see below) we have:
    \begin{gather*}
    |J(\pi_1)- J(\pi_2)| \leq 2r_{max} \sum^T_{t=0} \gamma^t \Big[  D_{TV} ( p^t_1(s,a) \| p^t_2(s,a) ) \Big] \leq \\
    \leq 2r_{max} \sum^T_{t=0} \gamma^t \Big[  D_{TV} ( p^t_1(s) \| p^t_2(s) ) + \mathbb{E}_{s \sim p} \big[ D_{TV} ( \pi_1(a|s) \| \pi_2(s,a) \big] \Big] 
    \end{gather*}
    Additionally, via Lemma \ref{lemma:TVDtime} (see below) we obtain:
    \begin{gather*}
    |J(\pi_1)- J(\pi_2)| \leq 2r_{max} \sum^T_{t=0} \gamma^t \Big[  D_{TV} ( p^t_1(s) \| p^t_2(s) ) + \mathbb{E}_{s \sim p} \big[ D_{TV} ( \pi_1(a|s) \| \pi_2(s,a) \big] \Big]  \leq \\
    \leq 2r_{max} \sum^T_{t=0} \gamma^t \Big[  t \delta + \mathbb{E}_{s \sim p} \big[ D_{TV} ( \pi_1(a|s) \| \pi_2(a|s) \big] \Big]  
    \end{gather*}
    From here, using Pinsker's inequality we derive:
    \begin{gather*}
    |J(\pi_1)- J(\pi_2)| \leq 2r_{max} \sum^T_{t=0} \gamma^t \Big[  t \delta + \mathbb{E}_{s \sim p} \big[ D_{TV} ( \pi_1(a|s) \| \pi_2(a|s) \big] \Big] \leq \\
    \leq 2r_{max} \sum^T_{t=0} \gamma^t \Big[  t \delta + \mathbb{E}_{s \sim p} \big[ \sqrt{2 D_{KL} ( \pi_1(a|s) \| \pi_2(a|s) )} \big] \Big]
    \end{gather*}
    Considering that we have two policies $\pi_1$ and $\pi_2$ that are 1-hidden layer MLPs, outputting a Gaussian distribution, by Lemma \ref{lemma:maskMLP} below we obtain:
    \begin{gather*}
    |J(\pi_1)- J(\pi_2)| \leq 2r_{max} \sum^T_{t=0} \gamma^t \Big[  t \delta + \mathbb{E}_{s \sim p} \big[ \sqrt{2 D_{KL} ( \pi_1(a|s) \| \pi_2(a|s) )} \big] \Big] \leq \\
    \leq 2r_{max} \sum^T_{t=0} \gamma^t \Big[  t \delta + \mathbb{E}_{s \sim p} \big[ \sqrt{2 \frac{(B_j s_j)^2}{2\sigma^2_f} )} \big] \Big] = \\
    = 2r_{max} \sum^T_{t=0} \gamma^t \Big[  t \delta +  \frac{B_j \mathbb{E}_{s \sim p} [|s_j|]}{\sigma_f}  \Big]
    \end{gather*}
    Finally, using the normalized inputs assumption by which $s_j \in [0,1]$ then we have:
    \begin{gather*}
    |J(\pi_1)- J(\pi_2)| \leq 2r_{max} \sum^T_{t=0} \gamma^t \Big[  t \delta +  \frac{B_j \mathbb{E}_{s \sim p} [|s_j|]}{\sigma_f}  \Big] \leq \\
    \leq 2r_{max} \sum^T_{t=0} \gamma^t \Big[  t \delta +  \frac{B_j}{\sigma_f}  \Big] \leq \\
    \leq 2r_{max} \sum^{\infty}_{t=0} \gamma^t \Big[  t \delta +  \frac{B_j}{\sigma_f}  \Big] \leq \\
    \leq 2r_{max}  \Big[  \frac{\gamma\delta}{(1 - \gamma)^2} +  \frac{B_j}{\sigma_f (1- \gamma)}  \Big] ~ ,
    \end{gather*}
    where $\frac{1}{2} \mathbb{E}_{s' \sim p} \big[ D_{TV} \big( p_1(s|s') \| p_2(s|s')  \big) \big] \leq \delta$ according to assumptions.
\end{proof}
\noindent After stating the full proof of Theorem \ref{thm:bound}, we proceed here below by providing statements of all the auxiliary lemmas instrumental to the theorem's proof, together with their own standalone proofs. For almost all the lemmas and proofs, we will be using the Total-Variation Distance on continuous probability spaces, which is defined for a continuous random variable $s \in \mathcal{S}$ and two density functions $p,q$ defined on the same probability space as:
\begin{gather*}
    D_{TV} \big( p(s) \| q(s) \big) = \frac{1}{2} \int_{s \in \mathcal{S}} | p(s) - q(s)| \, ds ~ .
\end{gather*}
\begin{lemma}[Joint Probability $D_{TV}$ Decomposition] \label{lemma:jointTVD}
    Consider two, time-dependent, joint state-action visitation probabilities $p^t_1 (s,a)$ and $p^t_2 (s,a)$, and their Total Variation Distance:
    \begin{gather*}
        D_{TV} \big( p^t_1 (s,a) \| p^t_2 (s,a) \big) = \frac{1}{2} \int_s \int_a | p^t_1 (s,a) - p^t_2 (s,a) | \, ds \, da ~ .
    \end{gather*}
    We can decompose this quantity into:
    \begin{gather*}
        D_{TV} \big( p^t_1 (s,a) \| p^t_2 (s,a) \big) \leq D_{TV} \big( p^t_1 (s) \| p^t_2 (s) \big) + \mathbb{E}_{s \sim p} \big[ D_{TV} \big( \pi_1 (a|s) \| \pi_2 (a |s) \big) \big] ~ .
    \end{gather*}
\end{lemma}
\begin{proof}[Proof of Lemma \ref{lemma:jointTVD}]
    Let us first define the recursive one-step decomposition of the joint probability $p^t_i (s,a)$:
    \begin{gather*}
        p^i_t (s,a) = \int_s p_i(s' |,s,a) \pi_i(a|s) p^{t-1}_i (s) \, ds ~ .
    \end{gather*}
    Then, we have that the difference:
    \begin{gather*}
        p^{t+1}_1 (s,a) - p^{t+1}_t (s,a) = \int_s \big[ p_1 (s' |s,a) \pi_1 (a|s) p^t_1(s) - p_2 (s' |s,a) \pi_2 (a|s) p^t_2(s) \big] \, ds = \\
        = \int_s \Big\{ p_1 (s' |s,a) \big[ \pi_1 - \pi_2 \big] p^t_1(s) + p_2(s'|s,a) \pi_2 \big[ p^t_1 (s) - p^t_2 (s) \big]  \Big\} \, ds
    \end{gather*}
    The Total Variation Distance between the two densities is then defined as:
    \begin{gather*}
        D_{TV} \big( p^{t+1}_1 (s,a) \| p^{t+1}_2 (s,a) \big) = \frac{1}{2} \int_s\int_a \left| p^{t+1}_1 (s,a) - p^{t+1}_t (s,a) \right| \, da \, ds \leq \\
        \leq \frac{1}{2} \int_s\int_a \left| p_1 (s' |s,a) \big[ \pi_1 - \pi_2 \big] p^t_1(s) \right| \, ds\,da \, + \frac{1}{2} \int_s\int_a \left| p_2(s'|s,a) \pi_2 \big[ p^t_1 (s) - p^t_2 (s) \right| \, ds\,da \,
    \end{gather*}
    For the first term of the sum above, we can derive:
    \begin{gather*}
        \frac{1}{2} \int_s\int_a \left| p_1 (s' |s,a) \big[ \pi_1 - \pi_2 \big] p^t_1(s) \right| \, ds\,da \leq \\
        \leq \frac{1}{2} \int_s\int_a \left|  \pi_1 - \pi_2 \right| p^t_1(s) \, ds\,da \leq \frac{1}{2} \int_s\int_a \left|  \pi_1 - \pi_2 \right| \, ds\,da \leq \\
        \int_s D_{TV} \big( \pi_1 \| \pi_2 \big) \, ds = \mathbb{E}_{s \sim p} \big[ D_{TV} \big( \pi_1 \| \pi_2 \big) \big]
    \end{gather*}
    For the second term instead we have:
    \begin{gather*}
        \frac{1}{2} \int_s\int_a \left| p_2(s'|s,a) \pi_2 \big[ p^t_1 (s) - p^t_2 (s) \right| \, ds\,da \leq \frac{1}{2} \int_s\int_a \left| p^t_1 (s) - p^t_2 (s) \right| \, ds\,da \leq \\
        \frac{1}{2} \int_s \left| p^t_1 (s) - p^t_2 (s) \right| \, ds  =  D_{TV} \big( p^t_1 (s) \| p^t_2 (s) \big)
    \end{gather*}
    Putting the two together, we eventually obtain:
    \begin{gather*}
        D_{TV} \big( p^{t+1}_1 (s,a) \| p^{t+1}_2 (s,a) \big) \leq D_{TV} \big( p^t_1 (s) \| p^t_2 (s) \big) + \mathbb{E}_{s \sim p} \big[ D_{TV} \big( \pi_1 (a|s) \| \pi_2 (a|s) \big) \big] ~ .
    \end{gather*}
\end{proof}

\begin{lemma}[Time-Dependent $D_{TV}$ Bound] \label{lemma:TVDtime}
    Suppose we assume same initial state distributions $p^0_1 (s) = p^0_2(s)$ and we define the following supremum quantity $\sup_t \mathbb{E}_{{\tilde{s}} \sim p} \big[ D_{TV} \big( p_1(s |{\tilde{s}}) \| p_2(s|{\tilde{s}}) \big) \big] \leq \delta$. Then we can obtain the following bound:
    $$ D_{TV} \big( p^t_1 (s) \| p^t_2 (s) \big) \leq t\delta ~ . $$
\end{lemma}

\begin{proof}[Proof of Lemma \ref{lemma:TVDtime}]
    Let us start by defining the one-step recursive decomposition:
    $$ p^t_1(s) = \int_{\tilde{s}} p_1(s|\tilde{s}) p^{t-1}_1 (\tilde{s}) \, d\tilde{s} ~ . $$
    Then we have:
    \begin{gather*}
        D_{TV} \big( p^t_1 (s) \| p^t_2 (s) \big) = \frac{1}{2} \int_s | p^t_1 (s) - p^t_2 (s) | \, ds = \\
        = \frac{1}{2} \int_{\tilde{s}} \int_{\tilde{s}} \left| p_1(s|\tilde{s}) p^{t-1}_1 (\tilde{s}) - p_2(s|\tilde{s}) p^{t-1}_2 (\tilde{s}) \right| \, ds \, d\tilde{s} \leq \\
        \leq \frac{1}{2} \int_{s} \left[ \int_{\tilde{s}} p^{t-1}_1 (\tilde{s}) \left| p_1(s|\tilde{s}) - p_2(s|\tilde{s}) \right| \,d\tilde{s} + \int_{\tilde{s}} p^{t-1}_2 (s|\tilde{s}) \left| p^{t-1}_1(s) - p^{t-1}_2(s) \right| \,d\tilde{s} \right] ds = \\
        = \frac{1}{2} \int_{\tilde{s}} p^{t-1}_1 (\tilde{s}) \left[ \int_s | p_1(s|\tilde{s}) - p_2(s|\tilde{s}) | \, ds \right] d\tilde{s} + \frac{1}{2} \int_{\tilde{s}} | p^{t-1}_1 (\tilde{s}) - p^{t-1}_2 (\tilde{s}) | \, d\tilde{s} = \\
        = \mathbb{E}_{\tilde{s} \sim p^{t-1}} \left[ \frac{1}{2} \int_s | p_1(s|\tilde{s}) - p_2(s|\tilde{s}) | \, ds \right] + D_{TV} \big( p^{t-1}_1 (s) \| p^{t-1}_2 (s) \big) \leq \\
        \leq \delta + D_{TV} \big( p^{t-1}_1 (s) \| p^{t-1}_2 (s) \big) ~ ,
    \end{gather*}
    By recursion over time $t$, we can compactly rewrite:
    $$ D_{TV} \big( p^{t}_1 (s) \| p^{t}_2 (s) \big) \leq t \delta + D_{TV} \big( p^{0}_1 (s) \| p^{0}_2 (s) \big) ~ , $$
    where by assumption $D_{TV} \big( p^{0}_1 (s) \| p^{0}_2 (s) \big) = 0$, leaving:
    $$ D_{TV} \big( p^{t}_1 (s) \| p^{t}_2 (s) \big) \leq t \delta ~ . $$
\end{proof}

\noindent Finally here below we prove the lemma regarding the masking of Multi-Layer Perceptron networks.

\begin{lemma}[Path Masking Intervention on MLPs] \label{lemma:maskMLP}
    Suppose we have a 1-hidden layer MLP that outputs $y \sim \mathcal{N} \big( f(x), \sigma^2_f \big)$ from inputs $(x_1, ..., x_d)$, where $\sigma^2_f \in \mathbb{R}^+$ and:
    $$ f(x) = \sum^H_{i=1} W_{2,i} \phi(z_i) + b_2 ~, ~~\text{ where } ~ z_i = \sum^d_{k=1} W_{1, ik} x_k + b_{1,i} ~ , $$
    and $\phi: \mathbb{R} \rightarrow (a,b) \subseteq \mathbb{R}$ is a $L_{\phi}$-Lipschitz continuous activation function. Then defining as $f_p(x)$ the “pruned” version of $f(x)$, where input $x_j$ is masked out via the intervention $W_{1,j} = 0$, we have that:
    $$ D_{KL} \big( \mathcal{N} \big( f(x), \sigma^2_f \big) \| \mathcal{N} \big( f_p(x), \sigma^2_f \big) \big) \leq \frac{(B_j |x_j|)^2}{2 \sigma^2_f} ~ , $$
    where $B_j = L_{\phi} \sum^H_{i=1} | W_{1, ij} W_{2,i} |$.
\end{lemma}

\begin{proof}[Proof of Proposition \ref{lemma:maskMLP}]
    The “pruned” version $f_p(x)$ of $f(x)$ can be written in the following form:
    $$ f_p(x) = \sum^H_{i=1} W_{2,i} \phi(z'_i) + b_2 ~, ~~~ \text{ where } ~ z'_i = \sum^d_{k=1, k\neq j} W_{1, ik} x_k + b_{1,i} ~ . $$
    Then, it is easy to obtain that $z'_i = z_i - W_{1, ij} x_j$, so that the difference becomes:
    $$ f(x) - f_p(x) = \sum^H_{i=1} W_{2,i} \left[ \phi(z_i) - \phi(z_i - W_{1, ij} x_j)   \right] ~ . $$
    Using the $L_{\phi}$-Lipschitz property of $\phi(\cdot)$ we obtain:
    \begin{gather*}
        |f(x) - f_p(x)| = | \sum^H_{i=1} W_{2,i} \left[ \phi(z_i) - \phi(z_i - W_{1, ij} x_j)   \right] | \leq \sum^H_{i=1} |W_{2,i}| \, | \phi(z_i) - \phi(z_i - W_{1, ij} x_j) | \leq \\
        \leq \sum^H_{i=1} |W_{2,i}| L_{\phi} |W_{1,ij} x_j| = |x_j| L_{\phi} \sum^H_{i=1} |W_{2,i} W_{1,ij}| = B_j |x_j| ~ .
    \end{gather*}
    Finally this implies that:
    $$ D_{KL} \big( \mathcal{N} \big( f(x), \sigma^2_f \big) \| \mathcal{N} \big( f_p(x), \sigma^2_f \big) \big) = \frac{\big( f(x) - f_p(x) \big)^2}{2 \sigma^2_f} \leq \frac{(B_j |x_j|)^2}{2 \sigma^2_f} ~ , $$
    where $D_{KL}$ is the KL divergence, defined as $D_{KL}(p\|q) = \int_{x \in \mathcal{X}} p(x) \log \frac{p(x)}{q(x)} \, dx$.
\end{proof}
\newpage

\subsection{Ablation Study}
\label{app:inf-ab}
We conduct ablation studies on InfrectroRL's hyperparameters: threshold ($\lambda$), amplification factor ($\gamma$), trigger size, and target label. For control, each hyperparameter is set to suboptimal values to observe episodic return deviations.
\subsubsection{Impact of $\gamma$}
We observe a consistent negative correlation between $\gamma$ and average reward (with $\lambda$, trigger size, and target label fixed). This trend is expected: higher $\gamma$ amplifies neurons, increasing the target action probability and consequently lowering episodic reward.
\begin{figure}[h!]
  \centering
  \begin{subfigure}[h!]{0.23\linewidth}
    \centering
    \includegraphics[width=\linewidth,height=0.23\textheight,keepaspectratio]{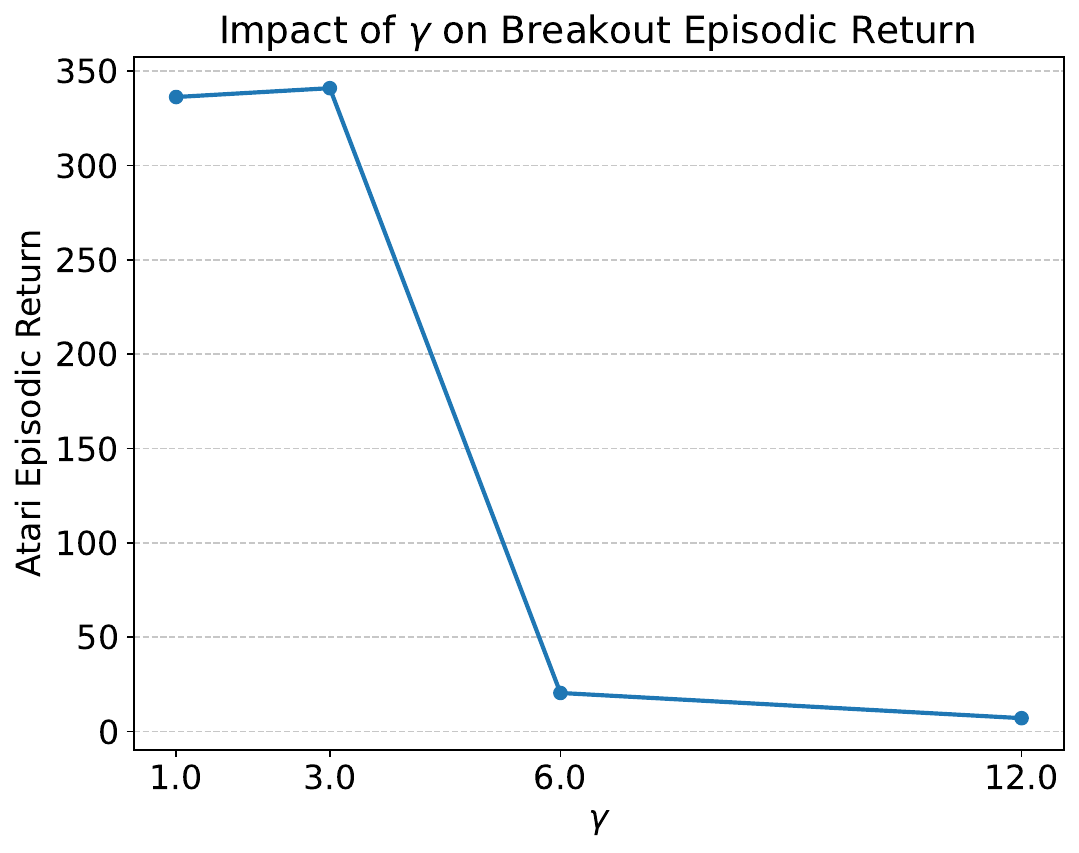}
    \caption{Breakout}
    \label{fig:gamma_breakout}
  \end{subfigure}
  \hspace{0.08\linewidth} 
  \begin{subfigure}[h!]{0.23\linewidth}
    \centering
    \includegraphics[width=\linewidth,height=0.23\textheight,keepaspectratio]{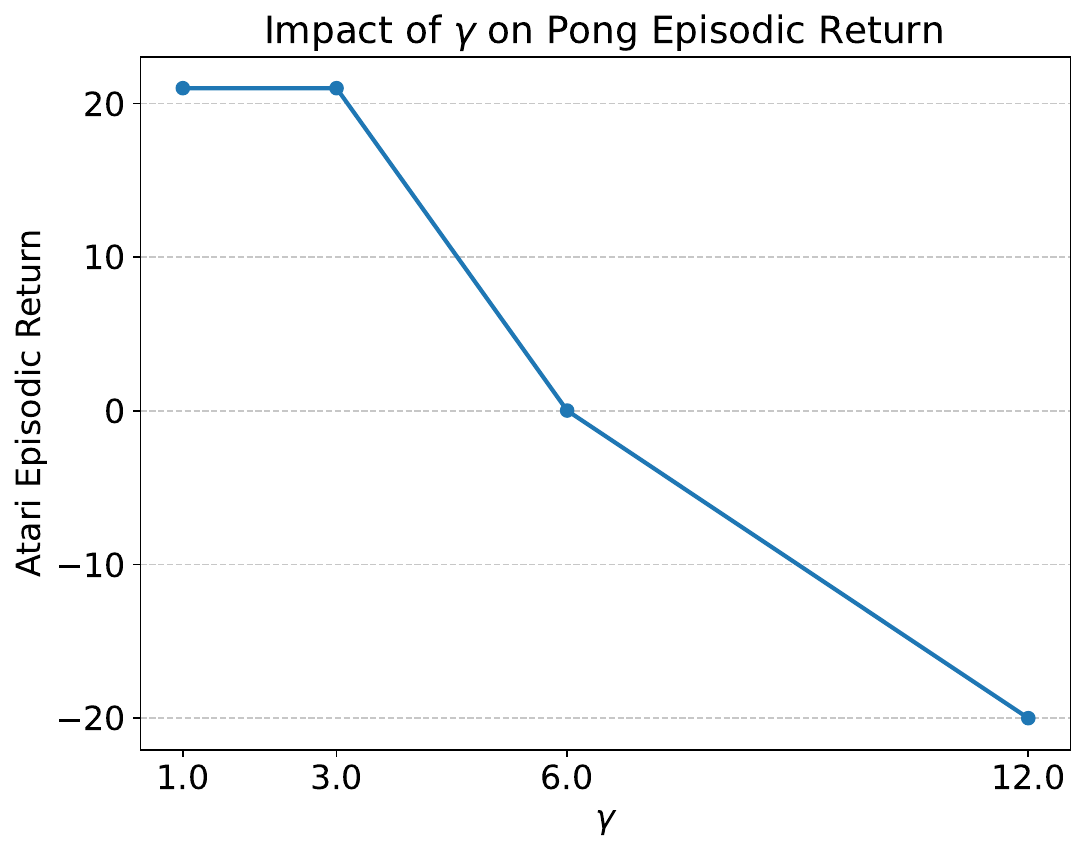}
    \caption{Pong}
    \label{fig:gamma_pong}
  \end{subfigure}
  
  \vspace{0.2ex} 
  
  \begin{subfigure}[h!]{0.23\linewidth}
    \centering
    \includegraphics[width=\linewidth,height=0.23\textheight,keepaspectratio]{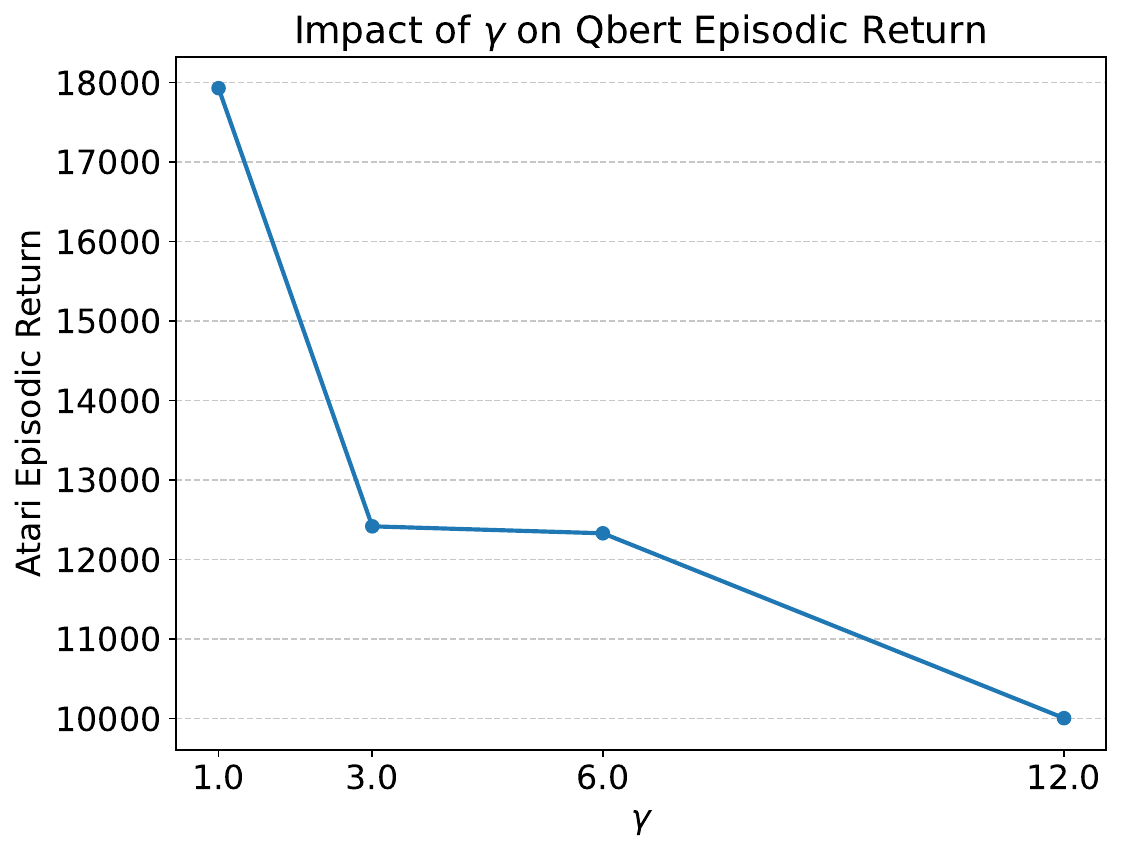}
    \caption{Qbert}
    \label{fig:gamma_qbert}
  \end{subfigure}
  \hspace{0.08\linewidth} 
  \begin{subfigure}[h!]{0.23\linewidth}
    \centering
    \includegraphics[width=\linewidth,height=0.23\textheight,keepaspectratio]{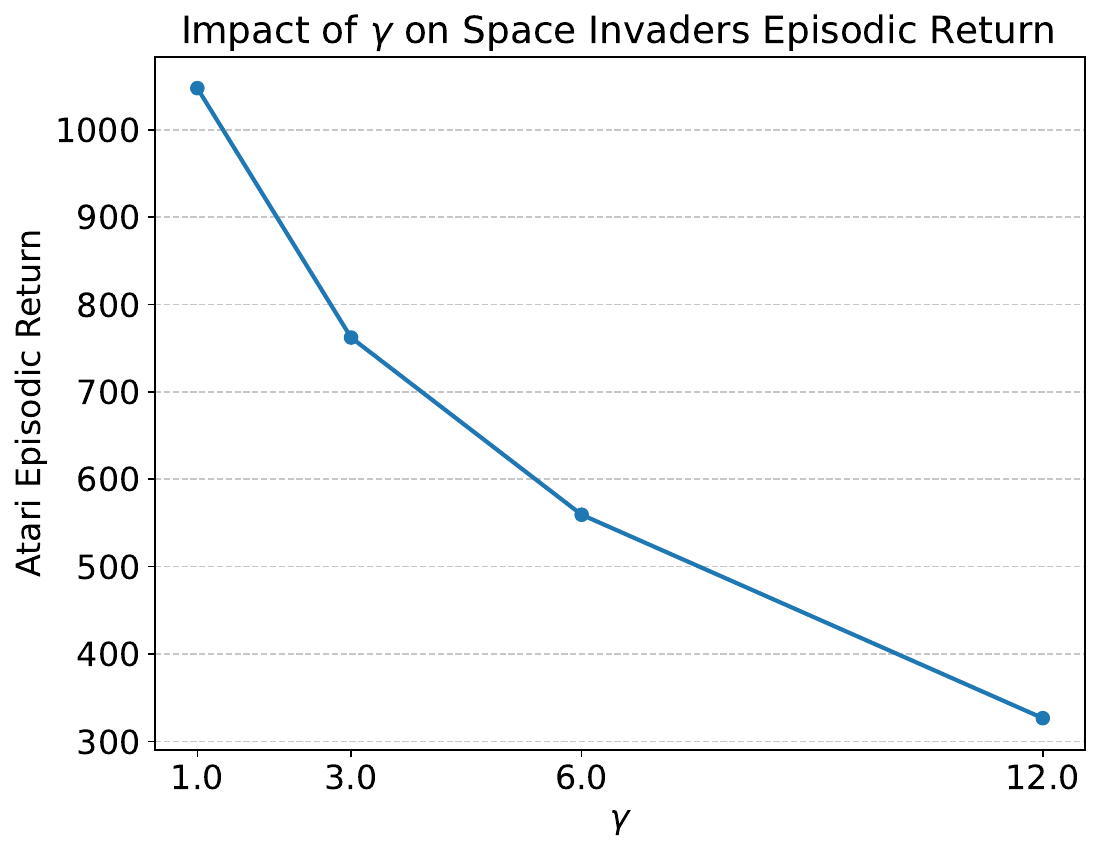}
    \caption{Space Invaders}
    \label{fig:gamma_spaceinv}
  \end{subfigure}
  
  \caption{Ablation study in InfrectroRL for varying $\gamma$ in all 4 games. We notice that varying $\gamma$ severely reduces the episodic return, as expected from our experimentations. $\gamma$ is known as the amplification factor that traverses through the policy network to induce a malicious adversary-chosen action.}
  \label{fig:gamma_fig}
\end{figure}


\subsubsection{Impact of $\lambda$}
We notice that our attack is maintains effective levels of episodic return regardless of the variation in $\lambda$. The reason is because the backdoor path is always activated for backdoored inputs where $\lambda$.
\begin{figure}[h!]
  \centering
  \begin{subfigure}[h!]{0.23\linewidth}
    \centering
    \includegraphics[width=\linewidth,height=0.23\textheight,keepaspectratio]{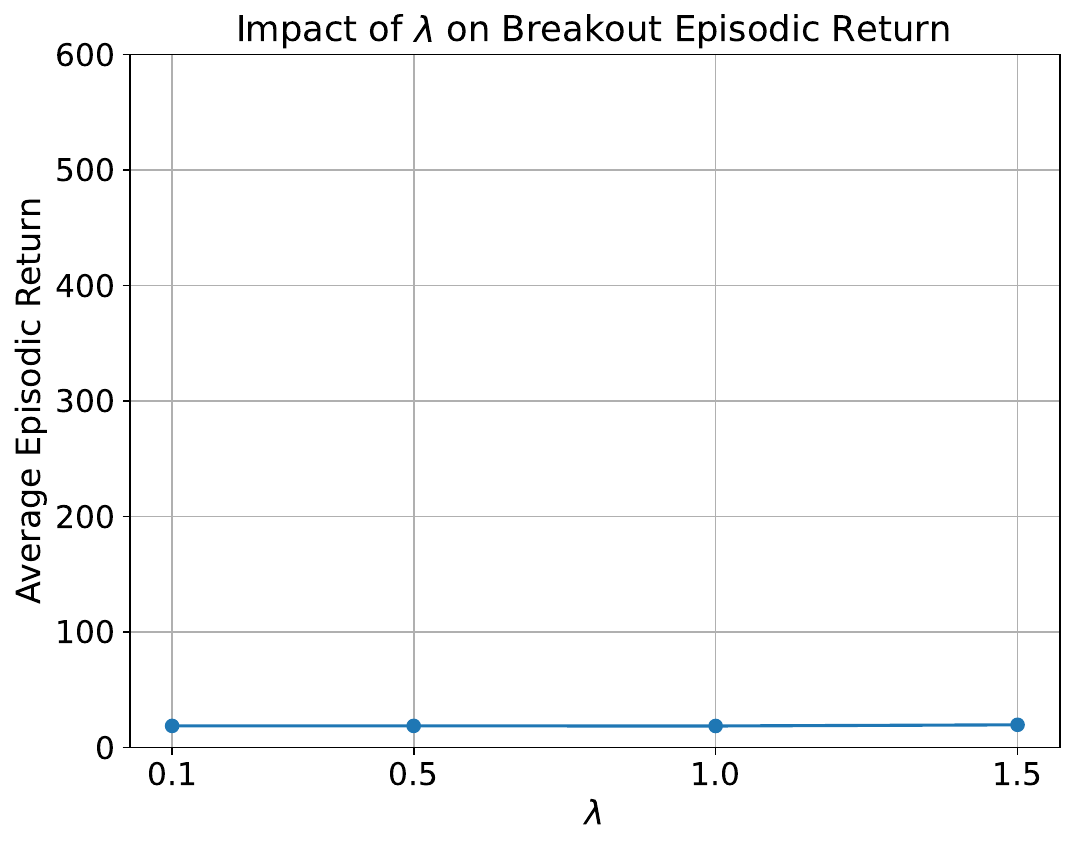}
    \caption{Breakout}
    \label{fig:lambda_breakout}
  \end{subfigure}
  \hspace{0.08\linewidth} 
  \begin{subfigure}[h!]{0.23\linewidth}
    \centering
    \includegraphics[width=\linewidth,height=0.23\textheight,keepaspectratio]{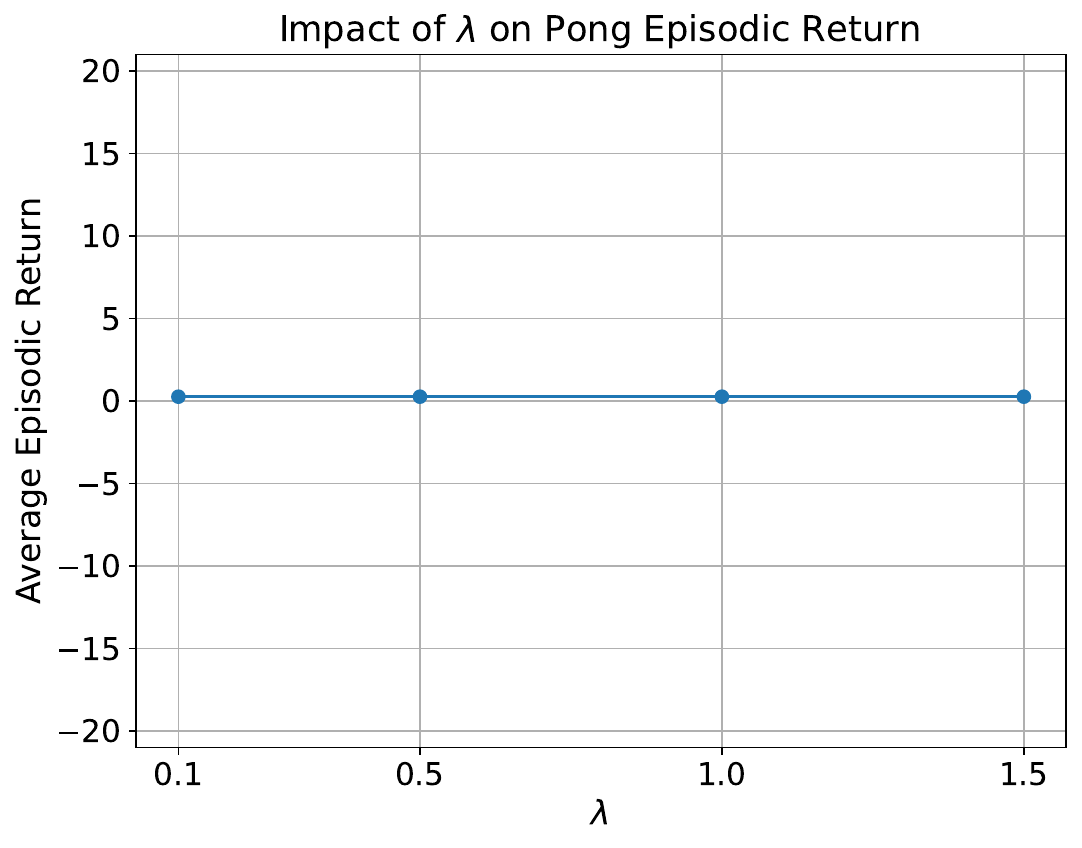}
    \caption{Pong}
    \label{fig:lambda_pong}
  \end{subfigure}
  
  \vspace{0.2ex} 
  
  \begin{subfigure}[h!]{0.23\linewidth}
    \centering
    \includegraphics[width=\linewidth,height=0.23\textheight,keepaspectratio]{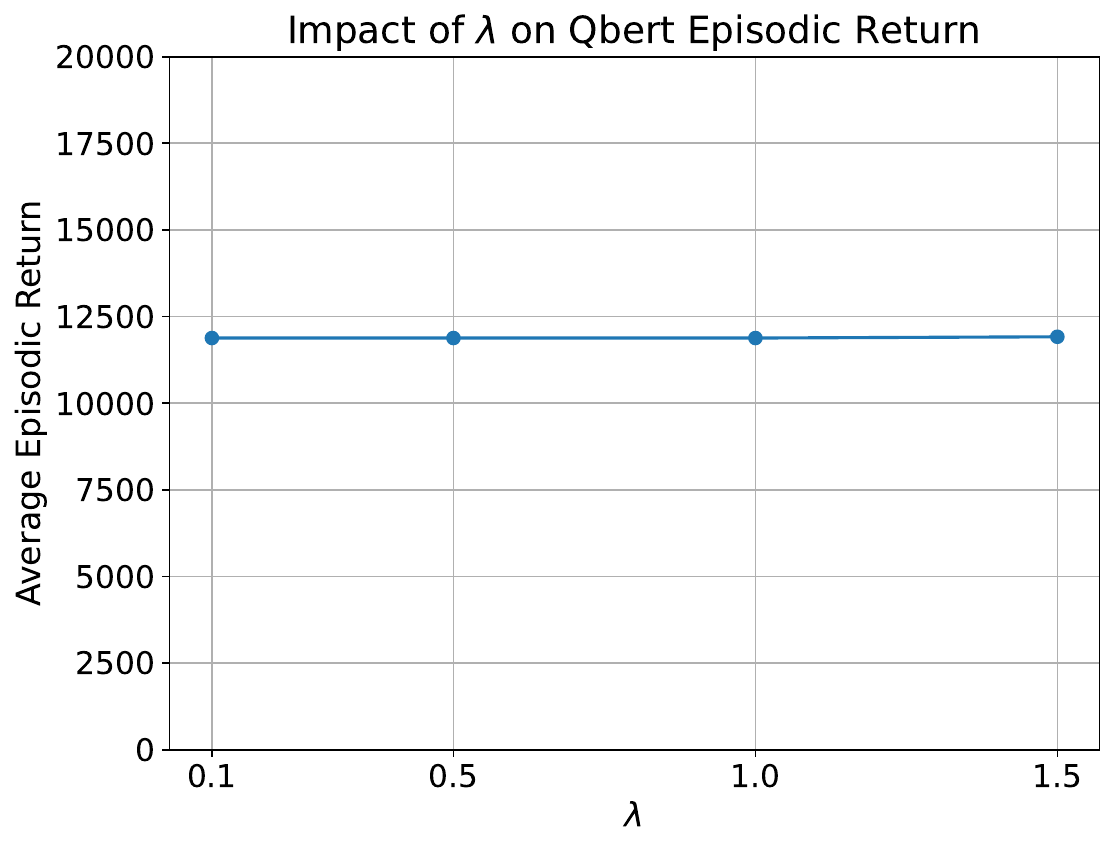}
    \caption{Qbert}
    \label{fig:lambda_qbert}
  \end{subfigure}
  \hspace{0.08\linewidth} 
  \begin{subfigure}[h!]{0.23\linewidth}
    \centering
    \includegraphics[width=\linewidth,height=0.23\textheight,keepaspectratio]{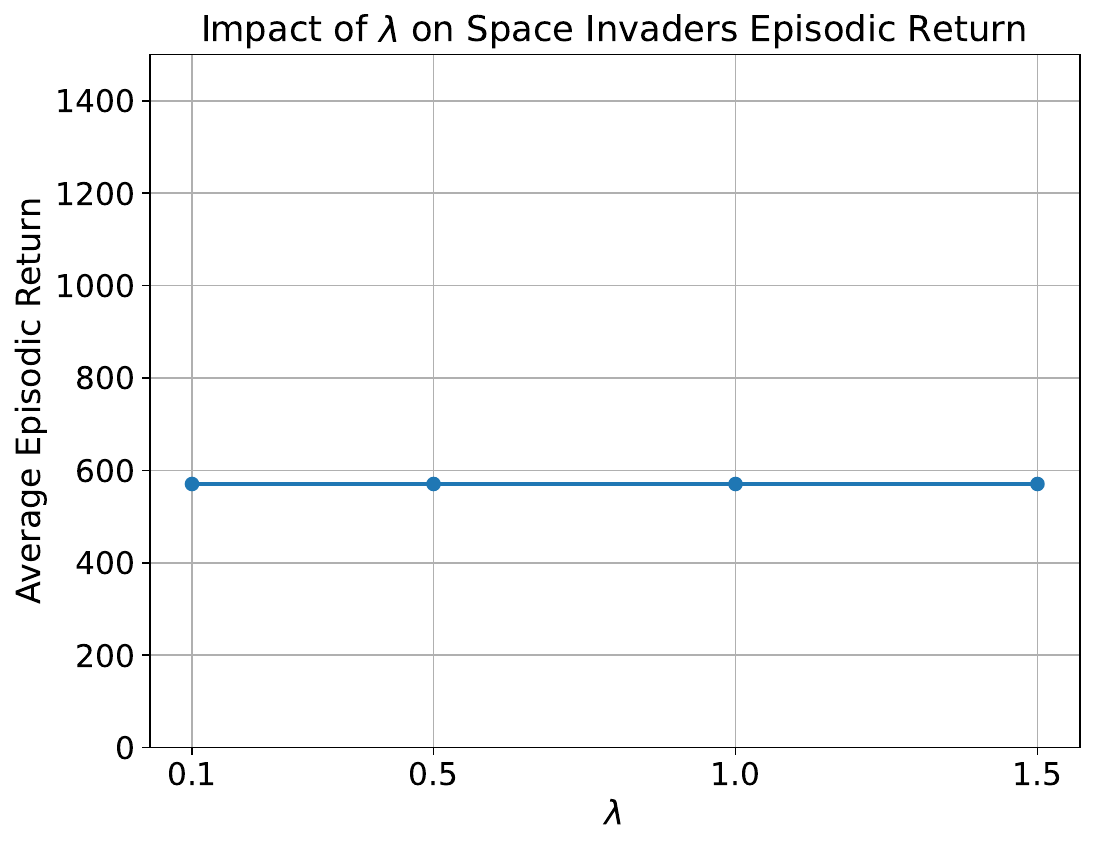}
    \caption{Space Invaders}
    \label{fig:lambda_spaceinv}
  \end{subfigure}
  
  \caption{Ablation study in InfrectroRL for varying $\lambda$ in all 4 games. We notice that varying $\lambda$ has no effect on the episodic return. The reason is that the backdoor path crafted by InfrectroRL is always activated for backdoored inputs when $\lambda >0 $.}
  \label{fig:lambda_fig}
\end{figure}

\newpage
\subsubsection{Impact of Trigger Size}
We notice that our attack consistently reduces the episodic reward regardless of the trigger size, making our attack resilient to change in appropriate trigger size.

\begin{figure}[h!]
  \centering
  \begin{subfigure}[h!]{0.23\linewidth}
    \centering
    \includegraphics[width=\linewidth,height=0.23\textheight,keepaspectratio]{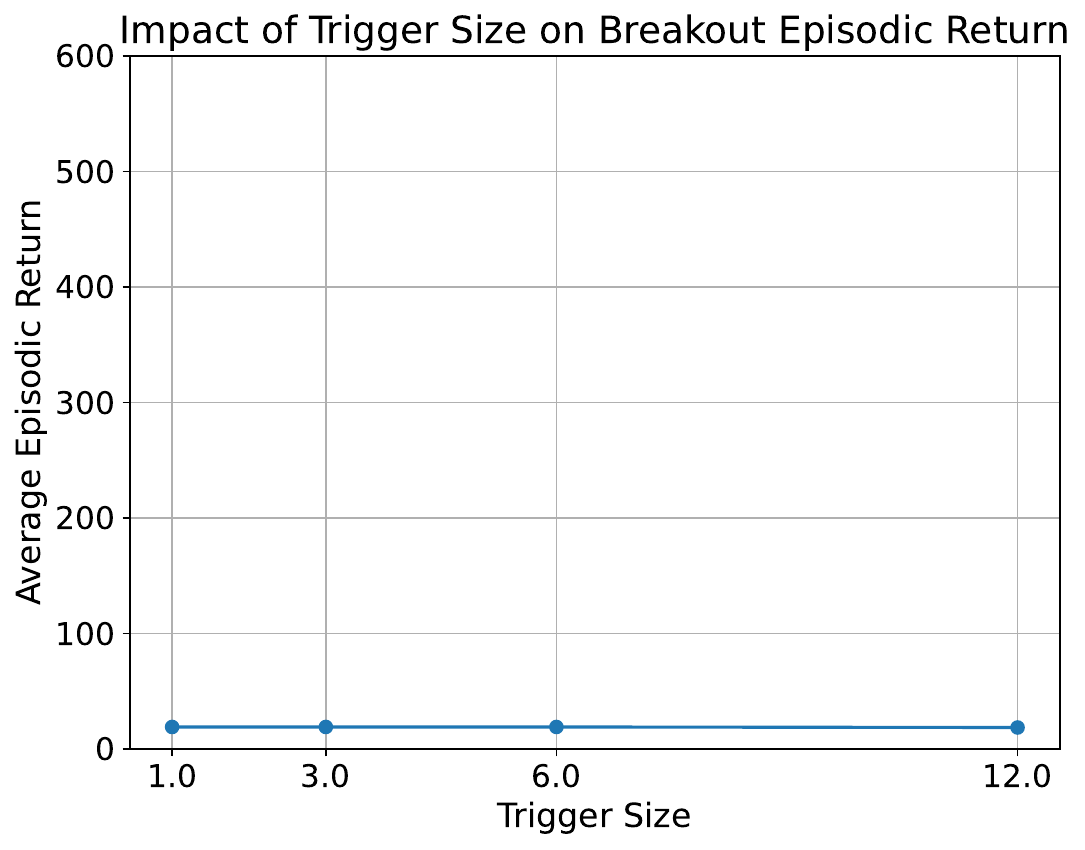}
    \caption{Breakout}
    \label{fig:trigger_breakout}
  \end{subfigure}
  \hspace{0.08\linewidth} 
  \begin{subfigure}[h!]{0.23\linewidth}
    \centering
    \includegraphics[width=\linewidth,height=0.23\textheight,keepaspectratio]{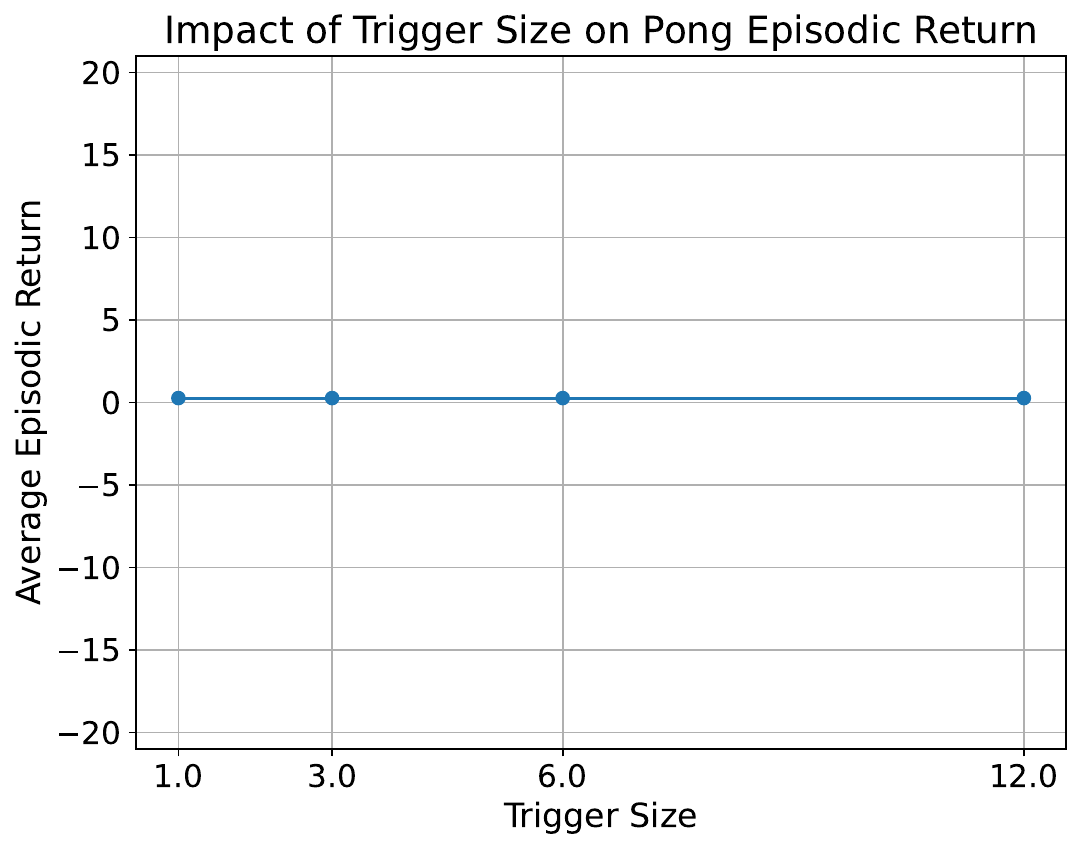}
    \caption{Pong}
    \label{fig:trigger_pong}
  \end{subfigure}
  
  \vspace{0.2ex} 
  
  \begin{subfigure}[h!]{0.23\linewidth}
    \centering
    \includegraphics[width=\linewidth,height=0.23\textheight,keepaspectratio]{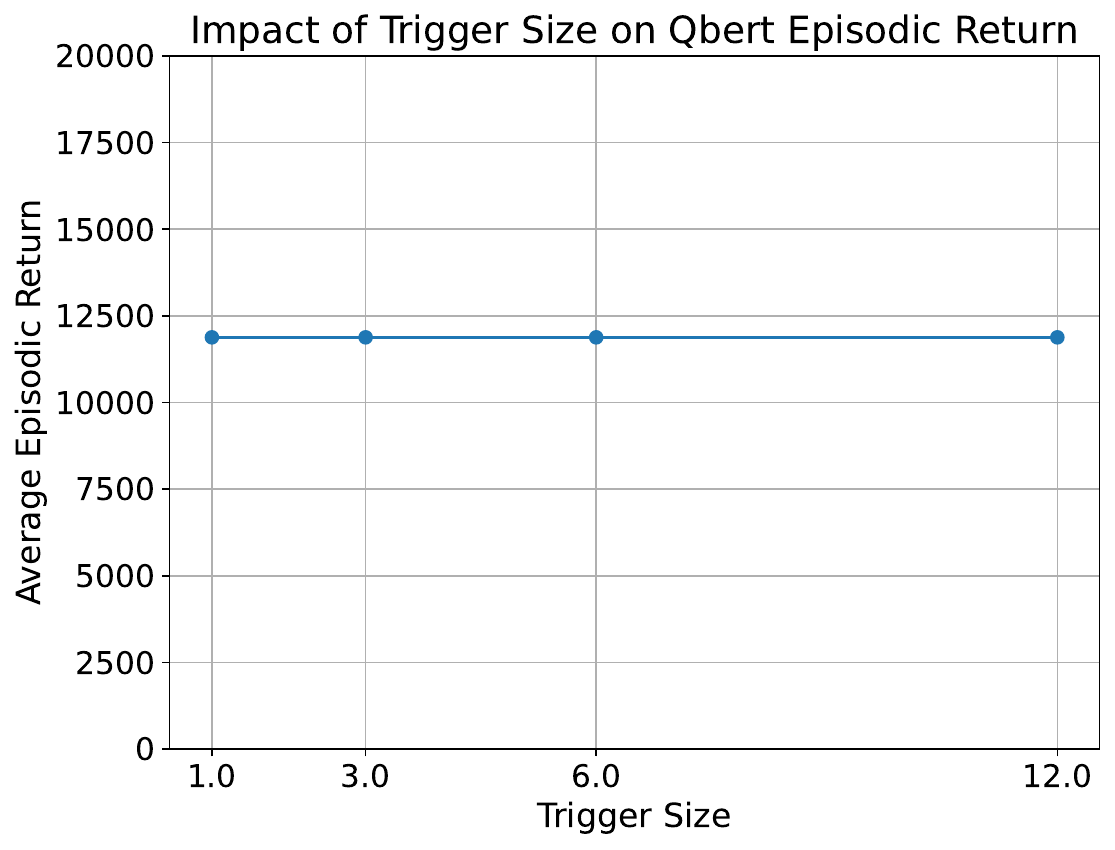}
    \caption{Qbert}
    \label{fig:trigger_qbert}
  \end{subfigure}
  \hspace{0.08\linewidth} 
  \begin{subfigure}[h!]{0.23\linewidth}
    \centering
    \includegraphics[width=\linewidth,height=0.23\textheight,keepaspectratio]{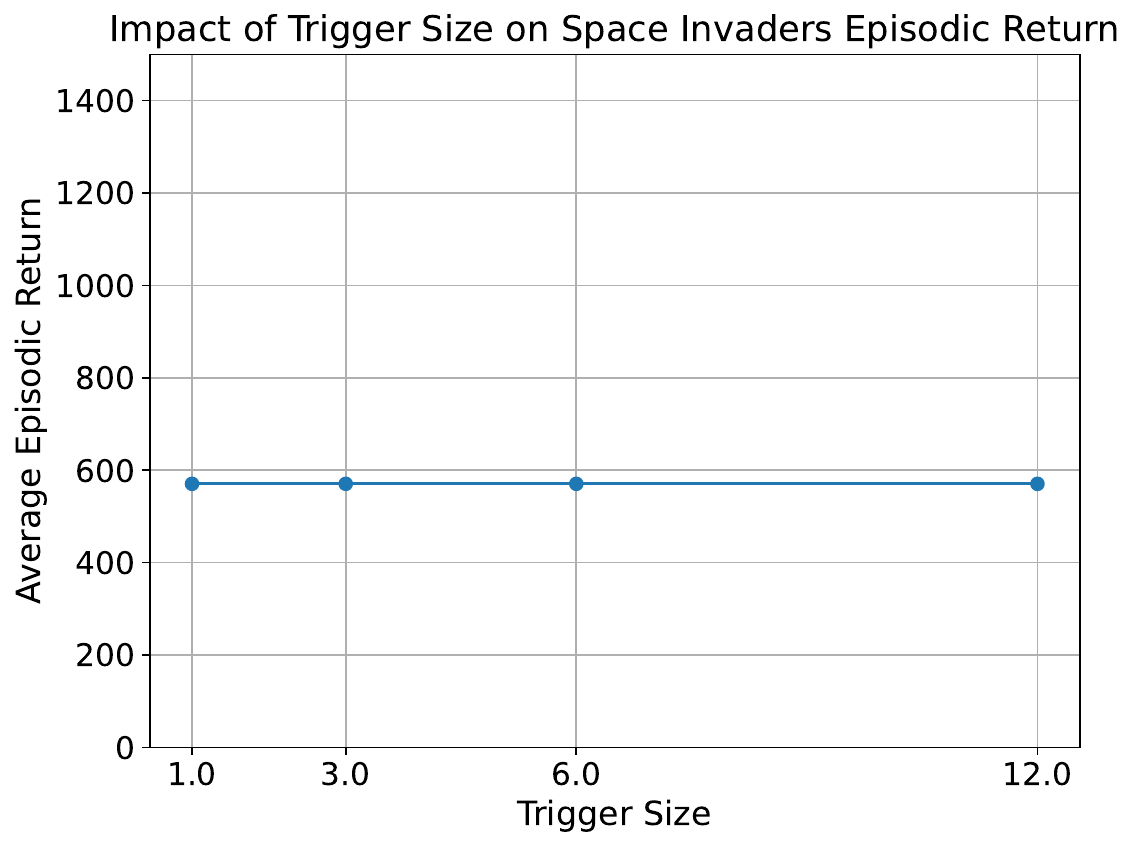}
    \caption{Space Invaders}
    \label{fig:trigger_spaceinv}
  \end{subfigure}
  
  \caption{Ablation study in InfrectroRL for varying the trigger size from 1 to 12 in all 4 games. We notice that InfrectroRL is highly effective regardless of the trigger size.}
  \label{fig:trigger_fig}
\end{figure}


\subsubsection{Impact of Target Label}
Given that our agent is a decision-making algorithm, we assume that the target actions can also affect the overall episodic return. As we set the hyperparameters, we set the remaining hyperparameters to a static values. We notice that all target label actions amount to similar levels of performance in the average episodic return. We assume this is primarily because of the repetitive actions made by the agent that cause it to reach a corner (or stay in the exact same position) in the environment.

\begin{figure}[h!]
  \centering
  \begin{subfigure}[h!]{0.23\linewidth}
    \centering
    \includegraphics[width=\linewidth,height=0.23\textheight,keepaspectratio]{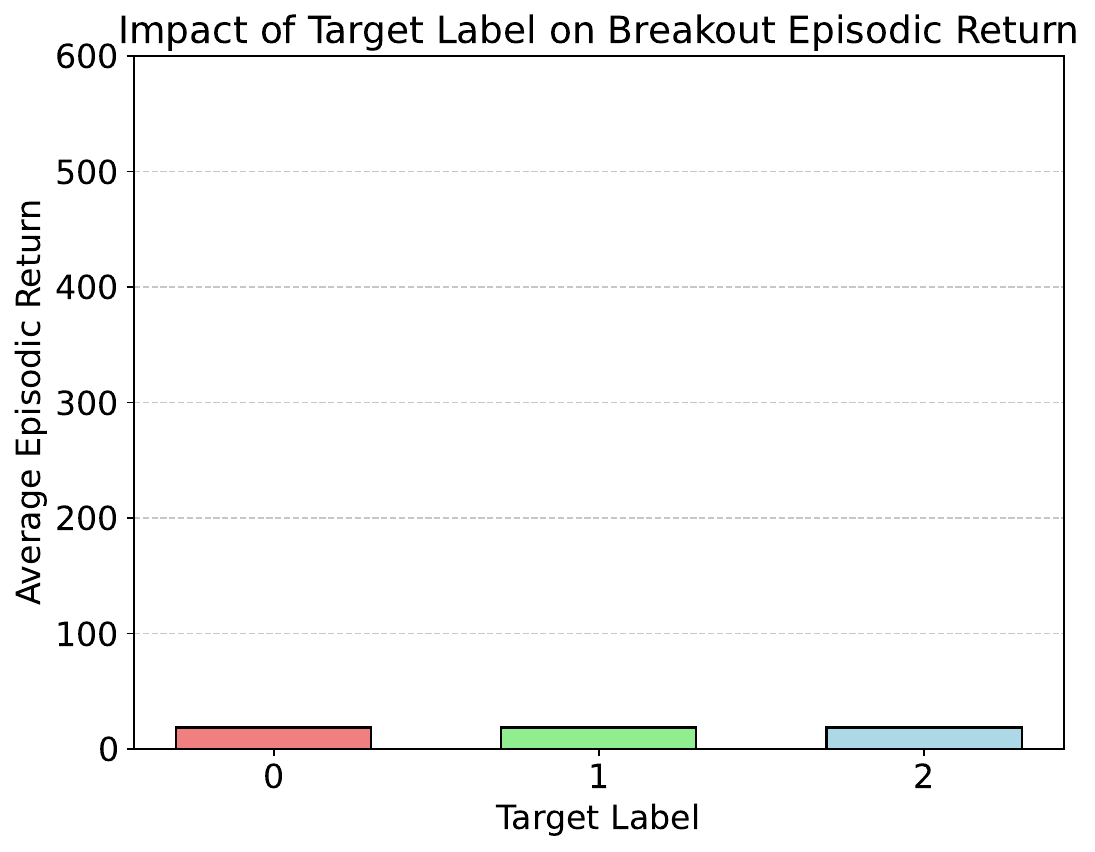}
    \caption{Breakout}
    \label{fig:target_breakout}
  \end{subfigure}
  \hspace{0.08\linewidth} 
  \begin{subfigure}[h!]{0.23\linewidth}
    \centering
    \includegraphics[width=\linewidth,height=0.23\textheight,keepaspectratio]{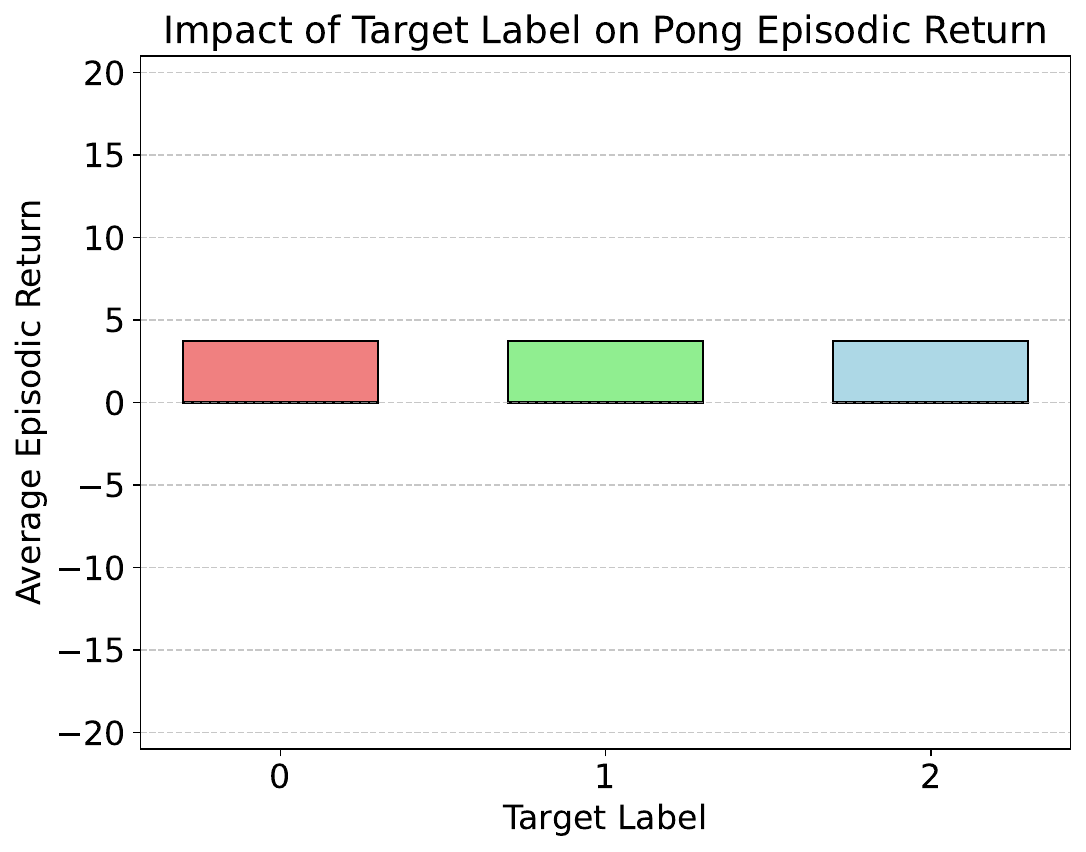}
    \caption{Pong}
    \label{fig:target_pong}
  \end{subfigure}
  
  \vspace{0.2ex} 
  
  \begin{subfigure}[h!]{0.23\linewidth}
    \centering
    \includegraphics[width=\linewidth,height=0.23\textheight,keepaspectratio]{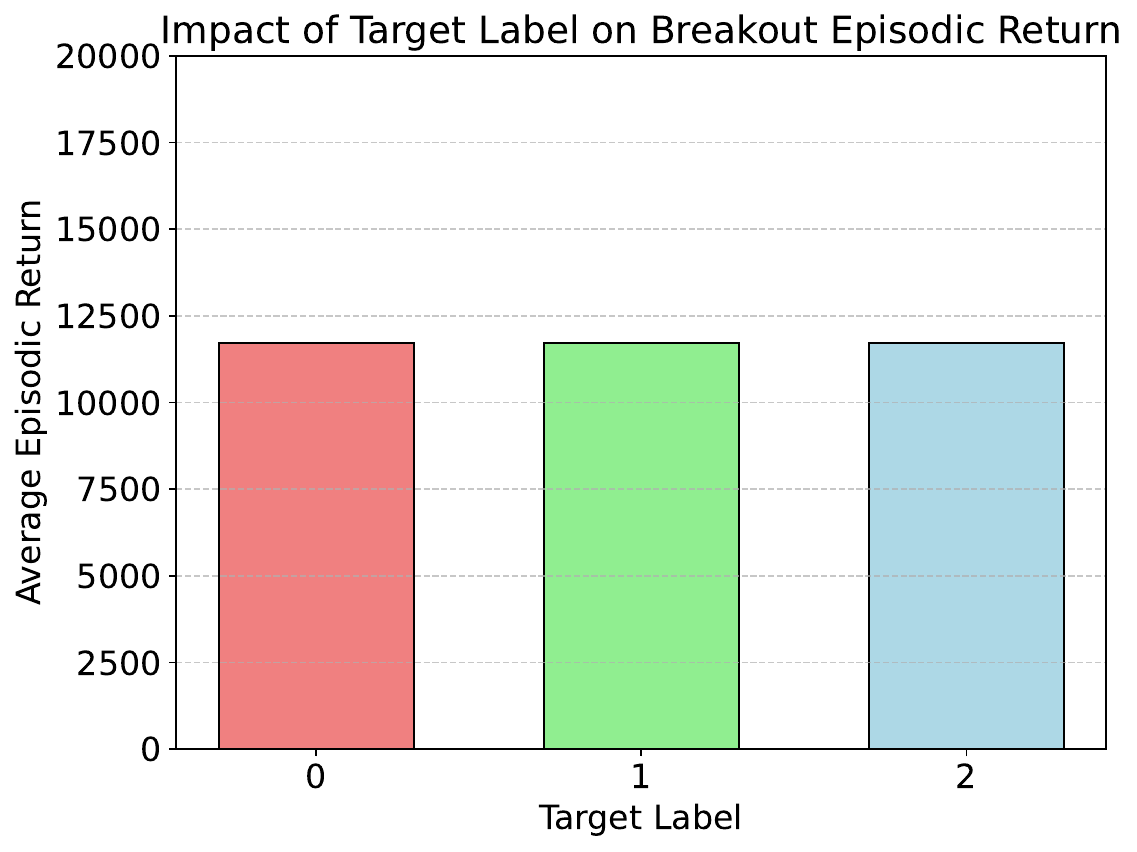}
    \caption{Qbert}
    \label{fig:target_qbert}
  \end{subfigure}
  \hspace{0.08\linewidth} 
  \begin{subfigure}[h!]{0.23\linewidth}
    \centering
    \includegraphics[width=\linewidth,height=0.23\textheight,keepaspectratio]{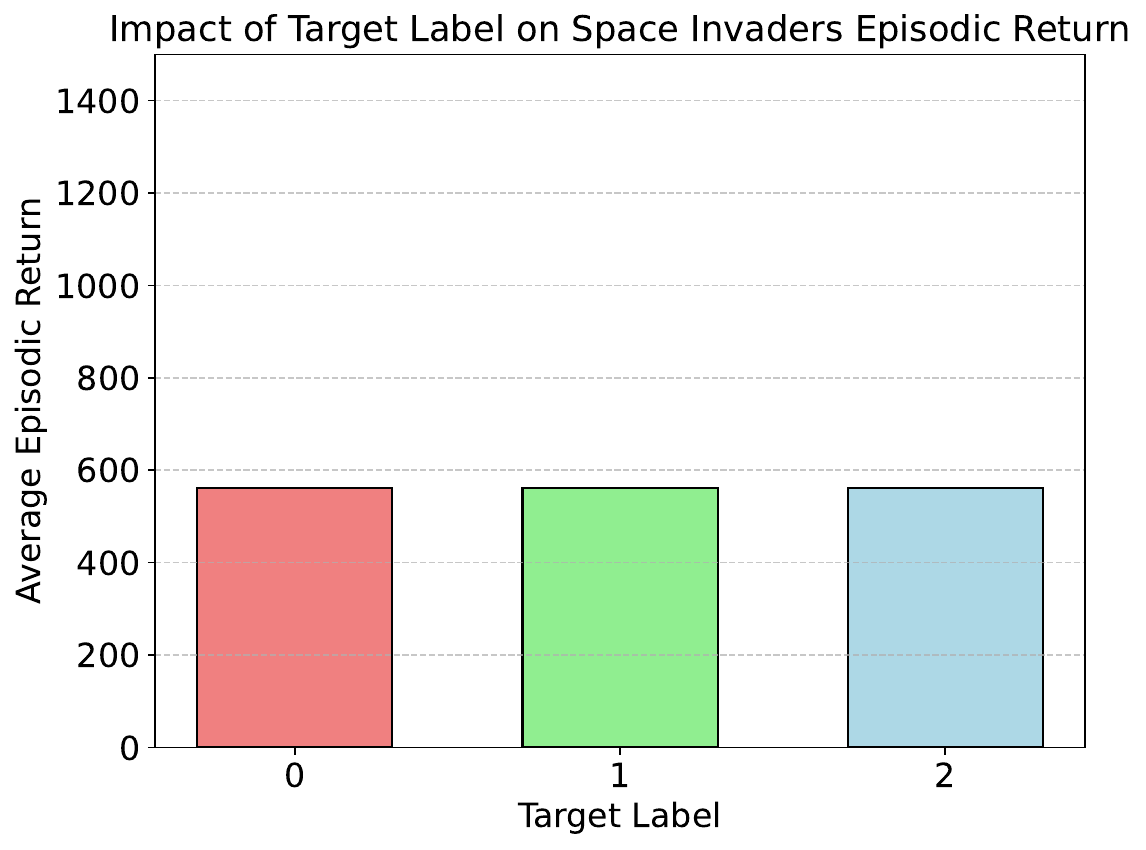}
    \caption{Space Invaders}
    \label{fig:target_spaceinv}
  \end{subfigure}
  
  \caption{Ablation study in InfrectroRL for varying the target label for all 4 games. We notice that the backdoor effects of InfrectroRL on the episodic return is resilient regardless of the target label.}
  \label{fig:target_fig}
\end{figure}

\end{document}